%% file: log_2025.tex
\documentclass{article}
\usepackage{log_2025}				% for anonymous submission to proceedings track

\usepackage{booktabs}						% professional-quality tables
\usepackage{multirow}						% tabular cells spanning multiple rows
\usepackage{amsfonts}						% blackboard math symbols
\usepackage{graphicx}						% figures
\usepackage{duckuments}						% sample images
\usepackage{todonotes}

\usepackage[numbers,compress,sort]{natbib}	% for numerical citations
\input{math_commands.tex}

\usepackage{macros}
\usepackage{figures}
\usepackage{microtype}
\usepackage{graphicx}
\usepackage{subfigure}
\usepackage{booktabs}
\usepackage{hyperref}
\usepackage{bbm}
\usepackage{makecell}
\usepackage{caption}

\newcommand{\xoffset}{-70}
\newcommand{\yoffset}{0}

\usepackage{amsmath}
\usepackage{amssymb}
\usepackage{mathtools}
\usepackage{amsthm}

\usepackage{hyperref}
\usepackage{url}

\title[AnyCQ: Graph Neural Networks for Answering  Queries over Incomplete Knowledge Graphs]{One Model, Any Conjunctive Query: Graph Neural Networks for Answering  Queries over Incomplete Knowledge Graphs}

\author[K. Olejniczak et al.]{
    Krzysztof Olejniczak \\
    University of Oxford \\
    \email{krzysztof.olejniczak@cs.ox.ac.uk}
    \And
    Xingyue Huang  \\
    University of Oxford \\
    \email{\qquad\qquad xingyue.huang@cs.ox.ac.uk\qquad\qquad}
    \AND
    Mikhail Galkin \\
    Intel AI \\
    \email{mikhail.galkin@intel.com}
    \And
    \.Ismail \.Ilkan Ceylan  \\
    TU Wien, AITHYRA, University of Oxford \\
    \email{ismail.ceylan@tuwien.ac.at}
}

\begin{document}

\maketitle

\begin{abstract}
Motivated by the incompleteness of modern knowledge graphs, a new setup for query answering has emerged, where the goal is to predict answers that do not necessarily appear in the knowledge graph, but are present in its \emph{completion}. In this paper, we formally introduce and study two query answering problems, namely, \emph{query answer classification} and \emph{query answer retrieval}.
To solve these problems, we~propose \(\anycq\), a model that can classify answers to~\emph{any} conjunctive query on \emph{any} knowledge graph.
At the core of our framework lies a graph neural network trained using a reinforcement learning objective to answer Boolean queries.
Trained only on simple, small instances, $\anycq$ generalizes to \emph{large queries} of \emph{arbitrary} structure, reliably classifying and retrieving answers to queries that existing approaches fail to handle.
This is empirically validated through our newly proposed, challenging benchmarks.
Finally, we empirically show that $\anycq$ can effectively transfer to \emph{completely novel} knowledge graphs when equipped with an appropriate link prediction model, highlighting its potential for querying incomplete data.
\end{abstract}

% \todo{Xingyue: I am labelling those I think need to be moved to appendix blue}

\section{Introduction}
Knowledge graphs (KGs) are an integral component of modern information management systems for \emph{storing}, \emph{processing}, and \emph{managing} data.
Informally, a KG is a finite collection of facts representing different relations between pairs of nodes.
% , which is typically highly incomplete. 
Motivated by the incompleteness of modern KGs~\citep{toutanova2015observed}, a~new setup for classical query answering has emerged~\citep{ren2020query2box, betae, gnn-qe, qto, fit, lmpnn, q2t}, where the goal is to predict answers that do not necessarily appear in the KG, but are potentially present in its \emph{completion}.
This task is~commonly referred to as complex query answering (CQA),
and poses a significant challenge, going beyond the capabilities of classical query engines, which typically assume every fact missing from the \emph{observable} KG is incorrect, following \emph{closed-world assumption} \citep{open_world_assumption}.

In its current form, CQA is formulated as a \emph{ranking} problem: given an input query \(Q(x)\) over a KG \(G\), the objective is to rank all possible answers based on their likelihood of being a correct answer.
Unfortunately, this setup suffers from various limitations.
Firstly, this evaluation becomes infeasible for cases where multiple free variables are allowed\footnote{As a result, almost all existing proposals focus on queries with only \emph{one} free variable.}. Moreover, to avoid explicitly enumerating solutions, existing methods need to resort to various heuristics and most of them can only handle tree-like queries \citep{cqd,qto,gnn-qe} or incur an exponential overhead in more general cases~\citep{fit}.
Consequently, the structural oversimplification of queries is also reflected in the existing benchmarks.
We argue for an alternative problem formulation, more aligned with classical setup, to alleviate these problems.

\textbf{Problem setup.} In this work, we deviate from the existing ranking-based setup, and instead propose and study two query answering problems based on classification. Our first task of interest, \emph{query answer classification}, involves classifying solutions to queries over  knowledge graphs, as~\(\mathsf{true}\) or~\(\mathsf{false}\). The second objective, \emph{query answer retrieval}, requires predicting \mbox{a correct answer} to~the~query or~deciding that none exists.

\textbf{Example.} Let us illustrate these tasks on a knowledge graph $G_\text{ex}$ (\Cref{fig:kg_and_query}), representing relationships between actors, movies, and locations. The dashed edges denote the missing facts from $G_\text{ex}$ and we write $\tilde{G}_\text{ex}$ to denote the complete version of $G_\text{ex}$ which additionally includes all missing facts.  Consider the following first-order query: 
\[
Q(x) = \exists y . \text{{Directed}}(y, \text{\textit{``Oppenheimer''}}) \land \text{{BornIn}}(y, x),
\]
which asks about the birthplace of the director of ``Oppenheimer''. 
\begin{itemize}
    \item \textbf{Query answer classification.} An instance of query answer classification is to classify a \emph{given} answer, such as \(x\rightarrow\text{\textit{London}}\), as~\(\mathsf{true}\) or \(\mathsf{false}\) based on the observed graph $G_\text{ex}$. In this case, the~answer \(x\rightarrow\text{\textit{London}}\) should be classified as $\mathsf{true}$, since this is a correct answer to $Q(x)$ in~the complete graph $\tilde{G}_\text{ex}$, whereas any other assignment should be classified as $\mathsf{false}$.
    \item \textbf{Query answer retrieval.} An instance of query answer retrieval is to predict a correct answer to $Q(x)$ based on the observed graph $G_\text{ex}$. In this case, the only correct answer is \(x\rightarrow\text{\textit{London}}\), which should be retrieved as an answer to the query $Q(x)$. If no correct answer exists, then \(\mathsf{None}\)
  should be returned as an answer.
\end{itemize}

\firstkgexample

\textbf{Approach and contributions.} To solve these tasks, we introduce \(\anycq\), a graph neural network that provided with a~function assessing the truth of unobserved links, can predict the satisfiability of a Boolean query over \emph{any} (incomplete) KG.
 \(\anycq\) acts as a search engine exploring the space of assignments to the free and existentially quantified variables in the query, eventually identifying a~satisfying assignment to the query.
$\anycq$ can handle \emph{any} existentially quantified first-order query in conjunctive or disjunctive normal form.
Our contributions can be summarized as follows:
\begin{enumerate}
    \item We extend the classical query answering problems to the domain of incomplete knowledge graphs and formally define the studied tasks of query answer classification and retrieval, introducing challenging benchmarks consisting of formulas with demanding structural complexity.
    \item We propose \(\anycq\), a neuro-symbolic framework for~answering Boolean conjunctive queries over incomplete KGs, which is able to solve existentially quantified queries of \emph{arbitrary} structure.
    \item We demonstrate the strength of \(\anycq\) on the studied objectives through various experiments, illustrating its strength on both benchmarks.
    \item Specifically, we highlight its surprising generalization properties, including transferability between different datasets and ability to extrapolate to very large queries, \emph{far beyond} the processing capabilities of existing query answering approaches.
\end{enumerate}

\section{Related work}
\looseness=-1

\textbf{Link prediction.} 
Earlier models for link prediction (LP) on knowledge graphs, such as TransE~\citep{transe}, RotatE~\citep{sun2019rotate}, ComplEx~\citep{trouillon2016complex}, and BoxE~\citep{abboud2020boxe}, learn fixed embeddings for entities and relations, confining themselves to \emph{transductive} setting. 
Later, graph neural networks (GNNs) emerged as powerful architectures, with prominent examples including RGCNs~\citep{schlichtkrull2018modeling} and CompGCNs~\citep{vashishth2020composition}.
{These models adapt the message-passing paradigm to multi-relational graphs, thus enabling \emph{inductive} link prediction on unseen entities.} 
Building on this, \citet{nbfnet} proposed NBFNets, which achieve strong performance through conditional message passing~\citep{huang2023a}.
Recently, \textsc{ULTRA}~\citep{galkin2023ultra} became one of the first foundation models on LP over both unseen entities and relations.

\textbf{Complex query answering.}
Complex query answering (CQA) \citep{ren2020query2box, betae} 
generalizes link prediction to 
first-order formulas with one free variable, considering queries with conjunctions (\(\land\)), disjunctions (\(\lor\)) and negations (\(\neg\)). 
\emph{Neuro-symbolic} models decompose the CQA task into a series of link prediction problems,
combining results with fuzzy logic.
CQD \citep{cqd} pioneered this approach with beam search over pre-trained embeddings for approximate inference. 
QTO \citep{qto} improved on this by exploiting the sparsity of neural score matrices to compute exact solutions without approximation. FIT~\citep{fit} extended QTO to cyclic queries at a higher cost.
GNN-QE \citep{gnn-qe} trained directly over queries without relying on pre-trained embeddings.
\textsc{UltraQuery} \citep{galkin2024ultraquery} combined GNN-QE’s framework with ULTRA, yielding the first foundation model for CQA with zero-shot generalization.
\emph{Neural} methods generally rely on neural networks to deduce relations and execute logical connectives simultaneously.
CQD-CO \citep{cqd} formulates query answering as continuous optimization, assigning embeddings to variables and optimizing the fuzzy logic with gradient descent.
LMPNN \citep{lmpnn} and CLMPT \citep{clmpt} employ logical message-passing and attention-based aggregation.
Q2T~\citep{q2t} utilized the adjacency matrix of the query graph as an attention mask in Transformers~\citep{attention_is_all} model.
While flexible, these methods lack interpretability and variable grounding, and underperform with growing query graph size.

\textbf{Combinatorial reasoning.}
GNNs have emerged as a powerful tool for solving combinatorial optimization problems \citep{gnn_co_survey}. Their power to leverage the inherent structural information encoded in graph representations of instances has been successfully utilized for solving various combinatorial tasks \citep{joshi, decision_col, hamiltonian_cycles}. As a method of our particular interest, ANYCSP \cite{anycsp}, introduced a novel computational graph representation for arbitrary constraint satisfaction problems (CSP), demonstrating state-of-the-art performance on~MAX-CUT, MAX-\textit{k}-SAT and \textit{k}-COL.

In this work, we cast conjunctive queries as a CSP, tailoring the ANYCSP framework to suit the task of satisfiability of Boolean formulas over incomplete KGs.
\(\anycq\) integrates link predictors to infer missing relations and introduces guidance mechanisms for efficient search over large domains.
Leveraging ANYCSP’s extrapolation and generalization strengths, \(\anycq\) provides an effective solution for query answer classification and retrieval.

\section{Preliminaries}
\label{sec:preliminaries}

\textbf{Knowledge graphs.} A \emph{knowledge graph} (KG) is a set of facts over a relational vocabulary $\sigma$, which is typically represented as a graph $G = (V(G), E(G), R(G))$, where $V(G)$ is the set of~nodes (or vertices), $R(G)$~is the set of~relation types, and $E(G) \subseteq R(G)\times V(G) \times V(G)$ is~the~set of~relational edges (i.e., facts), denoted as $r(u,v) \in E(G)$ with $r \in R(G)$ and $u,v \in V(G)$.
We~write $G \models r(a,b) $ to mean $ r(a,b) \in E(G)$.
We consider each given KG \({G} = (V(G),E({G}), R(G))\) as an \emph{observable part} of a~complete~graph \(\Tilde{G} = (V(G),E(\Tilde{G}), R(G))\) that consists of all true facts between entities in \(V(G)\). 
Under this assumption, reasoning over the known facts \(E(G)\) is insufficient, requiring deducing the~missing edges \(E(\Tilde{G})\backslash E(G)\).
Note that this formulation follows the transductive scenario, in which  \(\Tilde{G}\) covers the same sets of entities and relation types as \(G\).

\textbf{Link predictor.}
We call a \emph{link predictor} for a KG \(G\) a function \(\pi : R(G) \times V(G) \times V(G) \rightarrow [0,1]\), where \(\pi(r,a,b)\) represents the probability of the atom \(r(a,b)\) being a fact in \(E(\Tilde{G})\).
The \emph{perfect link predictor} \(\Tilde{\pi}\) for \(\Tilde{G}\) is defined as $\Tilde{\pi}(r,a,b) = 1$ if $r(a,b) \in E(\Tilde{G})$, and $0$ otherwise.

\textbf{First-order logic.}
A~\emph{term} is either a constant or a variable. A (binary) \emph{atom} is an expression of the form $r(t_1, t_2)$, where $r$ is a binary relation, and $t_1, t_2$ are terms.
A~\emph{fact}, or a \emph{ground atom}, has only constants as terms.
A~\emph{literal} is an atom or its negation.
A variable in a formula is \emph{quantified} (or \emph{bound}) if it is in the scope of \mbox{a quantifier}; otherwise, it is \emph{free}.
A \emph{Boolean formula} has no free variables.
\mbox{A \emph{quantifier-free formula}} does not use quantifiers. 
We write \(\vec{x} = x_1, ..., x_k\) and \(\vec{y} = y_1, ..., y_l \) to represent sequences of~variables \mbox{and \(\Phi(\vec{x},\vec{y})\)} to represent a quantifier-free formula \(\Phi\) using variables from $\{\vec{x},\vec{y}\}$.
Similarly, we write \(\vec{a}\) to represent tuples of constants of the form \mbox{\(\vec{a} = a_1, ..., a_k\)}.
For~a~first-order logic formula \(\Phi(\vec{x})\) with \(k\) free variables, we use \(\Phi(\vec{a}/\vec{x})\) to represent the Boolean formula obtained by substitution of each free occurrence of \(x_i\) for \(a_i\), for all \(i\).

\textbf{Query answering.} \label{conj_query_answering}
A \emph{conjunctive query (CQ)} is a first-order formula of the form \(Q(\vec{x})=\exists\vec{y}\,\Phi(\vec{x},\vec{y})\), where  \(\Phi(\vec{x},\vec{y})\) is a conjunction of literals using variables from \(\{\vec{x},\vec{y}\}\). We reserve \(\{\vec{y}\}\) for existentially quantified variables and \(\{\vec{x}\}\) for free variables. If the query is Boolean, we write \(Q=\exists\vec{y}\,\Phi(\vec{y})\). 
Given a KG \(G\) and a query  \(Q(\vec{x})=\exists\vec{y}\,\Phi(\vec{x},\vec{y})\), the~assignments \(\nu: \{\vec{x}\} \to V(G)\), 
\(\mu: \{\vec{y}\} \to V(G)\) respectively map the \emph{free} and \emph{quantified} variables to constants. 
For notational convenience, we denote with \(\vec{x} \to {\vec{a}}\) the~assignment \({x}_1\to {a}_1\), \ldots, \({x}_k\to {a}_k\). We represent by \(\Phi(\vec{a}/\vec{x},\vec{e}/\vec{y})\) the formula obtained by substituting the~variables with constants according to the assignments \(\vec{x} \to {\vec{a}}\) and \(\vec{y} \to \vec{e}\).
We write \(\nu_{x\rightarrow a}\) for an assignment such that \(\nu_{x\rightarrow a}(x) = a\) and \(\nu_{x\rightarrow a}(z) = \nu(z)\) whenever \(z\neq x\).

A Boolean query  \(Q=\exists\vec{y}\,\Phi(\vec{y})\) evaluates to \(\mathsf{true}\) on \(G\), denoted \(G \models Q\), if there exists an~assignment \(\vec{y} \to \vec{e}\) such that all positive facts that appear in \(\Phi(\vec{e}/\vec{y})\), appear~in~the~set \(E(G)\) and none of the negated facts that appear in \(\Phi(\vec{e}/\vec{y})\) are present in \(E(G)\). 
In~this~case, the assignment \(\vec{y} \to \vec{e}\) is called a \emph{match}. 
For a query \(Q(\vec{x})=\exists\vec{y}\,\Phi(\vec{x},\vec{y})\), an assignment \(\vec{x}\to\vec{a}\) is called an \emph{answer} if \({G} \models Q(\vec{a}/\vec{x})\).
We distinguish between \emph{easy} and \emph{hard} answers. An answer \(\vec{a}\) is \emph{easy} (or \emph{trivial}) if \(G \models Q(\vec{a}/\vec{x})\); it is \emph{hard} (or \emph{non-trivial}) if \(\Tilde{G} \models Q(\vec{a}/\vec{x})\) but \(G \nvDash Q(\vec{a}/\vec{x})\).

\textbf{Query graphs.}
Given a conjunctive query $Q(\vec{x})$, its \emph{query graph} has the terms of $Q(\vec{x})$ as vertices, and the atoms of $Q(\vec{x})$ as~relational~edges.
If the underlying undirected version of the resulting query graph is a tree, we call the query \emph{tree-like}, otherwise, we say it is \emph{cyclic}.

\textbf{Fuzzy logic.} Fuzzy logic extends Boolean Logic by introducing continuous truth values. A formula \(Q\) is assigned a truth value in range $[0,1]$, evaluated recursively on the structure of \(Q\) using \textit{t}-norms and \textit{t}-conorms.
In~particular, G\"odel \textit{t}-norm is defined as \(\top_G(a,b) = \min(a,b)\) with the~corresponding \textit{t}-conorm \(\bot_G(a,b) = \max(a,b)\).
For any existential Boolean formulas $Q$ and $Q'$, the~respective \emph{Boolean formula score} \(S_{\pi, G}\), w.r.t. a link predictor \(\pi\) over a KG \(G\) is then evaluated recursively as:
\[
\begin{array}{cc}
\begin{gathered}
\begin{aligned}
S_{\pi,G}(r(a,b)) &= \pi(r,a,b)&\qquad  \\
S_{\pi,G}(\neg Q) &= 1 - S_{\pi,G}(Q)&\qquad
\end{aligned}
\end{gathered}
&
\begin{gathered}
\begin{aligned}
&S_{\pi,G}(Q \land Q') = \min(S_{\pi,G}(Q),S_{\pi,G}(Q')) \\
&S_{\pi,G}(Q \lor Q') = \max(S_{\pi,G}(Q),S_{\pi,G}(Q')) \\
&S_{\pi,G}(\exists x.\,Q'(x)) = \max_{a\in V(G)} S_{\pi,G}(Q'(a/x))
\end{aligned}\\[-2pt]
\end{gathered}
\end{array}
\]

\section{Query answering on incomplete KGs}
\label{sec:problem_formulation}

Existing problem formulations for complex query answering (CQA) suffer from several fundamental limitations that restrict directions for progress.
First, the standard ranking-based objective is \textbf{computationally infeasible} for queries with \textit{multiple free variables}, as it requires scoring all candidate answers.
Their number is exponential in the number of free variables, which is a known bottleneck for large knowledge graphs \citep{fb15k237, nell, efok_dataset}.
This has also led to reliance on benchmarks with \textbf{limited structural complexity}, featuring either tree-like queries \citep{betae} or ones with a single cycle \citep{fit}, failing to capture the richness of real-world reasoning tasks.
Finally, many ranking-based models are \textbf{not probabilistically calibrated} for binary classification \citep{cqa_classification}. Their scores often result from non-probabilistic training objectives \citep{lmpnn, clmpt} or require ad-hoc transformations \citep{qto, fit}, making them unsuitable for applications that demand definitive true/false decisions. See extended discussion in \Cref{app:limitation}.

\subsection{Query Answer Classification \& Query Answer Retrieval}
In this section, we propose two new query answering tasks designed to provide more targeted responses to complex logical queries, while avoiding the complexity incurred by answer enumeration.

\emph{Query answer classification} reflects real-world scenarios where users seek to verify the correctness of~a~specific answer rather than navigating through a ranked list of possibilities.
It better captures the~nature of many real-world queries, aligning the model’s output with the user’s intent:
\:\\
\begin{mdframed}[style=MyFrame]
\label{problem:qac}
\textsc{Query Answer Classification (QAC)}

\textbf{Input}: A query $Q(\vec{x})$, tuple \(\vec{a}\) and an observed graph $G$.\newline
\textbf{Output}: Does \(\Tilde{G} \models Q(\vec{a}/\vec{x})\) hold?
\end{mdframed}

\newpage

\emph{Query answer retrieval} assesses the correctness of the top-ranked result. By requiring models to~either deliver a correct {assignment to the free variables of the input query} or assert the absence of one, QAR aligns more closely with practical decision-making, ensuring the output is relevant~and~reliable:\\

\begin{mdframed}[style=MyFrame]
\label{problem:qar}
\textsc{Query Answer Retrieval (QAR)}

\textbf{Input}: A query $Q(\vec{x})$ and an observed graph $G$.\newline
\textbf{Output}: $\vec{x}\to \vec{a}$ where $\tilde{G} \models Q(\vec{a}/\vec{x})$ or \(\mathsf{None}\)
\end{mdframed}

\textbf{Reduction to conjunctive query answering.} 
As shown in prior work \citep{lmpnn, fit, clmpt}, any existentially quantified first-order (EFO) query \(Q(\vec{x})\) can be rewritten in disjunctive normal form
\mbox{\(
Q'(\vec{x}) = \exists \vec{y}\,(D_1(\vec{x},\vec{y}) \lor \cdots \lor D_n(\vec{x},\vec{y})),
\)}
where each \(D_i\) is a conjunction of literals. Thus, we have \(Q'(\vec{x}) \equiv Q_1(\vec{x}) \lor \cdots \lor Q_n(\vec{x})\) with
\(Q_i(\vec{x}) = \exists \vec{y}\,D_i(\vec{x},\vec{y})\), and \(Q(\vec{a}/\vec{x})\) is satisfiable if and only if some \(Q_i(\vec{a}/\vec{x})\) is satisfiable, reducing the task to conjunctive queries.
Unlike ranking, which must combine scores across disjuncts, classification aggregates only binary outcomes.
Hence we focus on conjunctive queries, as solving each \(Q_i\) individually introduces no additional complexity.

\section{\anycq: framework for query answering }

{To address the introduced tasks of query answer classification and retrieval}, we propose a neuro-symbolic framework for scoring arbitrary existential Boolean formulas, called \(\anycq\).
Let \(\pi\) be a link predictor for an observable knowledge graph \(G\).
An~\(\anycq\) model \(\Theta\) equipped with \(\pi\) can be viewed as a function
\(
\Theta(G,\pi) : \mathsf{CQ^0}(G) \rightarrow [0,1] 
\)
where \(\mathsf{CQ^0}(G)\) is~the~class of~conjunctive Boolean queries over the same vocabulary as $G$.
For input \(Q = \exists \vec{y} . \Phi(\vec{y})\), \(\Theta\)~searches over the space of assignments to \(\vec{y}\) for
\[
\alpha_{\max} = \argmax_{\alpha : \vec{y} \rightarrow V(G)} S_{\pi,G}(\Phi(\alpha(\vec{y}) / \vec{y})).
\]
and returns an approximation \(\Theta(Q | G, \pi)\) of \(S_{\pi,G}(Q)\) as
\[
\Theta(Q | G, \pi) = \max_{\text{visited } \alpha} S_{\pi,G}(\Phi(\alpha(\vec{y})/\vec{y})) 
\]
Note that by unfolding the Boolean formula score:
\(
    S_{\pi,G}(Q) = S_{\pi,G}(\Phi(\alpha_{\max}(\vec{y})/\vec{y})) \approx \Theta(Q | G, \pi).
\)
Hence, by leveraging the potential of GNNs for solving combinatorial optimization problems, we can recover strong candidates for \(\alpha_{\max}\), allowing for an accurate estimation of \(S_{\pi,G}(Q)\).

\textbf{Overview.} During the search, our method encodes the query \(Q\) and its relation to the current assignment \(\alpha\) into a computational graph \(G_{Q,\alpha}\) (\Cref{sec:computational_graph}).
This graph is then processed with a simple GNN \(\theta\) (whose architecture is described in \Cref{app:architecture}), which updates its hidden embeddings and generates distributions \(\mu\) from which the next assignment \(\alpha'\) is sampled (\Cref{sec:search_process_framework}). 

\subsection{Query representation}
\label{sec:computational_graph}

We transform the input queries into a computational graph (\Cref{fig:computational_graph}), whose structure is adopted from ANYCSP \citep{anycsp}.
Consider a conjunctive Boolean query
\(Q=\exists\vec{y}. \Phi(\vec{y})\) over a knowledge~graph~\(G\), with~\(\Phi\) quantifier-free, and let \(\pi\) be a link predictor for~\(G\). Let \(c_1, ..., c_n\) be constant symbols mentioned in \(\Phi\), and \(\psi_1, ..., \psi_l\) be the~literals in~\(\Phi\). 
We define the domain \(\gD(e)\) of the term \(e\) as \(\gD(y) = V(G)\) for each existentially quantified variable \(y\) and \(\gD(c_i) = \{c_i\}\) for~each~constant~\(c_i\).
Given an assignment \(\alpha: \vec{y}\rightarrow\vec{a}\), the computational graph \(G_{Q,\alpha}\) is constructed as follows:

\figcompgraphremasteredtest

\textbf{Vertices.} 
The vertices of \(G_{Q,\alpha}\) are divided into three groups. Firstly, the \emph{term nodes}, \(v_{y_1}, ..., v_{y_k}\) and \(v_{c_1}, ..., v_{c_n}\), represent variables and constants mentioned in \(\Phi\).
Secondly, \emph{value vertices} correspond to feasible term-value assignments.
Formally, for each term \(e\) mentioned in \(\Phi\) and any value \(a\in\gD(e)\), there exists a value vertex \(v_{e\rightarrow a}\).
Finally, \emph{literal nodes} \(v_{\psi_1}, \dots, v_{\psi_l}\) represent literals  \(\psi_1, ..., \psi_l\) of \(\Phi\).

\textbf{Edges.} We distinguish two types of edges in \(G_{Q,\alpha}\). The~\emph{term-value} edges connect term with value nodes: for any term vertex \(v_e\) representing \(e\) and any \(a\in\gD(e)\), there exists an undirected edge \(\{v_e, v_{e\rightarrow a}\}\).
Additionally, \emph{value-literal} edges are introduced to propagate information within literals. If a literal \(\psi_i\) mentions a term \(e\), then for all \(a\in\gD(e)\) there is an~edge between \(v_{\psi_i}\) and \(v_{e\rightarrow a}\).

\textbf{Edge labels.}
Edge labels embed the predictions of the link predictor \(\pi\) into the computational graph \(G_{Q,\alpha}\) to support guided search. Each value-literal edge connecting a literal vertex \(v_{\psi_i}\) with a~value node \(v_{e\rightarrow a}\) is annotated with the \emph{potential edge} (PE) and the \emph{light edge} (LE) labels.
The~PE label \(P_E(v_{\psi_i}, v_{e\rightarrow a})\) is meant to answer the question: ``Can \(\psi_i\) be satisfied  under the substitution \(e \rightarrow a\)?''.
For example, when \(\psi_2 = s(x,y)\), as in \Cref{fig:computational_graph}, \(P_E(v_{\psi_2}, v_{x,a})\) denotes whether \(\exists y . s(a,y)\) is satisfiable, according to \(\pi\).
We pre-compute the~PE labels using \(\pi\), binarizing the Boolean formula scores of~the~form~\(S_{\pi,G}(\exists y. s(a,y))\) with the threshold \(0.5\). 

In contrast to PE labels, which are independent of the assignment \(\alpha\), light edge (LE) labels reflect how local changes to \(\alpha\) affect satisfiability of the literals.
Formally, we set \(L_E(v_{\psi_i}, v_{e\rightarrow a}; \alpha)=1\) if \(\psi_i\) is satisfied under the assignment \(\alpha_{z\rightarrow a}\), and \(0\) otherwise. In other words, LE labels answer the~question: ``If we change \(\alpha\) so that \(z\) is assigned to \(a\), will \(\psi_i\) be satisfied?''.
Satisfiability is again determined by binarizing the prediction score returned by the link predictor \(\pi\).
Hence, through these edge labels, \(\pi\) effectively guides the search toward promising updates in the assignment space. 

Further explanations of edge labels are provided in \Cref{app:edge_labels}.

\subsection{\(\anycq\) search process}
\label{sec:search_process_framework}
The outline of the search process conducted by \(\anycq\) is presented in \Cref{fig:anycq_overview}.
Before the~search commences, the hidden embeddings \(\mathbf{h}^{(0)}\) of all \emph{value nodes} are set to a~pre-trained vector \(\mathbf{h}\in\mathbb{R}^d\) and an initial assignment \(\alpha^{(0)}\) is~drafted, sampling the value for each variable \(y\in\{\vec{y}\}\) independently and uniformly at random from \(\mathcal{D}(y)\).
The variable and literal nodes are not assigned any hidden embeddings, serving as intermediate steps for value node embedding updates.  
At the beginning of search step \(t\), \(G_{Q,\alpha^{(t-1)}}\) is~processed with a GNN \(\theta\), which generates new value node embeddings \(\mathbf{h}^{(t)}\), and for each variable \(y\in\vec{y}\) returns a~distribution \(\mu_y^{(t)}\) over \(\mathcal{D}(y)\).
Finally, the next assignment \(\alpha^{(t)}\) is sampled by drawing the value \(\alpha^{(t)}(y)\) from \(\mu_y^{(t)}\), independently for each \(y\in\{\vec{y}\}\).
A~precise description of the architecture of \(\theta\) is provided in Appendix \ref{app:architecture}. The search terminates after \(T\) steps,
producing assignments \(\alpha^{(0)}, \alpha^{(1)}, ..., \alpha^{(T)}\) which are used to approximate \(S_{\pi,G}(Q)\) as
\[
\Theta(Q | G, \pi) = \max_{0 \leq t \leq T} S_{\pi,G}\left(
\Phi\left(\alpha^{(t)}(\vec{y}) / \vec{y}\right)
\right)
\]

\subsection{Training}
During training on each dataset, we equip \(\anycq\) with a~predictor \(\pi_\text{train}\), representing the training graph \(G_\text{train}\).
Thus, the only trainable component of \(\Theta\) remains the GNN~\(\theta\). 
We utilize the training splits from the existing CQA datasets \cite{betae}, hence limiting the scope of queries viewed during training to formulas mentioning at most three variables.
Moreover, we restrict the number of search steps $T$ to at most \(15\), encouraging the network to quickly learn to apply logical principles locally.
Inspired by prior work on combinatorial optimization \citep{neuraltsp, anycsp, Abe2019SolvingNP}, we~train \(\theta\) in a~reinforcement learning setting via REINFORCE \cite{reinforce}, treating \(\theta\) as a search policy network with~the~objective of~maximizing \(\Theta(Q | G, \pi)\).
This setup enables \(\anycq\) to generalize across different query types, scaling to formulas of size several times larger than observed during training.
See complete methodology in~\Cref{app:anycq_training}.

\anycqoverviewfigremastered
\subsection{Theoretical and conceptual properties}
\label{sec:anycq_properties}
Our \(\anycq\) framework is supported by strong theoretical guarantees and designed for broad conceptual flexibility. Theoretically, the method is provably \textbf{complete}, ensuring it converges to the correct answer with sufficient search steps (\Cref{thm:completeness}), and it is \textbf{sound} when equipped with a perfect link predictor, guaranteeing the correctness of its positive predictions (\Cref{thm:soundness}). Conceptually, \(\anycq\) is built for \textbf{transferability} and \textbf{generality}. Because its core model is independent of the input graph and link predictor, it can be seamlessly applied to unseen knowledge graphs, as we have shown in \Cref{sec:ablation_studies}. Furthermore, its general design allows it to handle relations of any arity and process complex formulas in conjunctive or disjunctive normal form (\Cref{app:scope_of_formulas}), even extending to inductive settings with the appropriate predictor (\Cref{app:choice_of_predictor}).

\section{Experimental evaluation}
We empirically evaluate \(\anycq\) to assess its performance on the proposed tasks of Query Answer Classification (QAC) and Query Answer Retrieval (QAR). To provide a comprehensive analysis, we aim to answer the following key questions:
\begin{itemize}
\item[\textbf{Q1}.] How does \(\anycq\) perform on QAC and QAR comparing with existing models? (\Cref{sec:qac_results,sec:qar_results})
\item[\textbf{Q2}.] How does \(\anycq\) perform outside the training domain? (\Cref{sec:outside_training_domain})
\item[\textbf{Q3}.] How does \(\anycq\) perform with a perfect link predictor? (\Cref{sec:ablation_studies})
\item[\textbf{Q4}.] How does the choice of link predictor affect \(\anycq\)'s overall performance? (\Cref{sec:used_complex_lp})
\item[\textbf{Q5}.] How does \(\anycq\)'s performance scale with increasing query complexity? (\Cref{app:extended_evaluation})
\end{itemize}

\subsection{Experimental setup}
\label{sec:experiment_setup}

\textbf{Benchmarks and datasets.} 
Existing CQA benchmarks \citep{betae, efok_dataset} contain mostly simple query structures, which impair development of more advanced and general methods.
To close this gap, we generate new datasets on top of standard benchmarks, introducing queries of higher structural complexity. These formulas may involve up to 20 distinct terms and feature multiple cycles, non-anchored leaves, long-range dependencies, and multi-way conjunctions. See \Cref{app:dataset_generation} for generation details.

For \textbf{QAC}, we propose FB15k-237-QAC and~\mbox{NELL-QAC}, each divided into 9 splits, consisting of small and large formulas.
We focus exclusively on single-variable instances, as multi-variable cases reduce trivially to the single-variable setting, i.e., \(\langle Q(x_1,x_2), (a_1, a_2), G\rangle\) is equivalent to a single-variable instance \(\langle Q(x_1, a_2/x_2), a_1, G\rangle\) as they both ask if \( \Tilde{G} \models Q(a_1/x_1, a_2/x_2)\).

For \textbf{QAR}, we observe that many instances of the simpler query structures, inherited from existing CQA benchmarks, admit easy answers, i.e.~have at least one satisfying assignment supported entirely by observed facts. Combined with their limited structural complexity, this makes them trivial under the QAR objective, which only requires recovering a~single correct answer. To evaluate reasoning under incompleteness and structural difficulty, we introduce new benchmarks: FB15k-237-QAR and NELL-QAR, consisting of large formulas with up to three free variables.

\textbf{Baselines.} As the baselines for the small-query split on our QAC task, we choose the state-of-the-art solutions from CQA capable of handling the classification objective: QTO \citep{qto}, FIT \citep{fit}, GNN-QE \citep{gnn-qe} and \textsc{UltraQuery} \citep{galkin2024ultraquery}.
Considering the large-query splits, we notice that no existing approaches can be applied in this setting, as none of them can simultaneously: 
1) efficiently handle \textbf{cyclic} queries and 2) produce \textbf{calibrated} probability estimates, without the knowledge of the trivial answers\footnote{See \Cref{app:baseline_discussion} for a detailed explanation.}.
Hence, we furthermore use an SQL engine, implemented by DuckDB \citep{duckdb}, reasoning over the observable graph.
{For~the~same reasons, extended by the need of reasoning over queries with multiple variables, we consider \emph{only} the SQL engine as the baseline for QAR experiments.}
In both cases, we limit the processing time to 60 seconds, ensuring termination in a reasonable time.
Additional evaluations ablating the impact of this timeout, using 30, 60, and 120 seconds thresholds, are included in \Cref{app:extended_evaluation}.
Training details for the considered baselines are provided in \Cref{sec:used_complex_lp}.

\textbf{Methodology.} Given a Boolean query \(Q\) over an observable KG \(G\), an \(\anycq\) model~\(\Theta\) equipped with a link predictor \(\pi\) for \(G\) can decide if \(\Tilde{G} \models Q\), by returning whether \(\Theta(Q | G, \pi) > 0.5\).
We~use this functionality to solve QAC instances by applying our \(\anycq\) models directly to \(Q(\vec{a}/\vec{x})\).
For~the~QAR task, given a query \(Q(\vec{x})\) over an observable KG \(G\), we run our \(\anycq\) framework on~the~Boolean formula \(\exists \vec{x} . Q(\vec{x})\), returning \(\mathsf{None}\) if the returned \(\Theta(\exists \vec{x} . Q(\vec{x}) | G, \pi)\) was less than \(0.5\).
Otherwise, we return \(\alpha(\vec{x})\) where \(\alpha\) is the visited assignment maximizing the Boolean formula score.
In both scenarios, we perform 200 search steps on each input instance in the large query splits, and just 20 steps for small QAC queries.
We equip NBFNet as the default link predictor for both QAC and QAR evaluations (details in \Cref{app:choice_of_predictor}).

\textbf{Metrics.} Given the classification nature of both our objectives, we use the 
F1-score 
as the metric for~query answer classification and retrieval (see \Cref{app:dataset_evaluation_methodology} for details). 
In QAR, we mark a~positive solution as correct only if the returned assignment is an answer to the input query.
In~contrast to~the~CQA evaluation, we also include easy answer (instances), since the~task of~efficiently answering queries with advanced structural complexity, even admitting answers in~the~observable knowledge graph, is not trivial.
The code can be found in \href{https://github.com/kolejnyy/ANYCQ}{this GitHub repository}.

\begin{table*}[t]
  \centering
  \vspace{-1em}
  \scriptsize
    \caption{Average F1-scores of considered methods on the QAC task.}
  \begin{tabular}{ccccccccccc}
    \toprule  
     \textbf{Dataset}&\textbf{Model} & \textbf{2p} & \textbf{3p} & \textbf{pi} & \textbf{ip} & \textbf{inp} & \textbf{pin} & \textbf{3-hub} & \textbf{4-hub} & \textbf{5-hub} \\
    \midrule
    \multirow{6}{*}{\textbf{FB15k-237-QAC}}
    &\textsc{SQL} & 66.0 & 61.7 & 70.0 & 67.0 & 78.1 & 74.8 & 37.0 & 32.2 & 35.3 \\
    &{QTO} & 67.1 & 64.4 & 70.8 & 67.7 & 78.5 & 75.9 & -- & -- & -- \\
    &{FIT} & 68.0 & 65.1 & 71.4 & 67.8 & 78.6 & 76.7 & -- & -- & --  \\
    &GNN-QE & \textbf{77.1} & \textbf{73.5} & 80.1 & \textbf{81.2} & \textbf{79.0} & 77.0 & -- & -- & -- \\
    &\textsc{UltraQuery} & 75.2 & 68.9 & 79.8 & 76.8 & 75.9 &\textbf{78.6} & -- & -- & -- \\
    &\(\anycq\)\:\: & 75.8 & 71.3 & \textbf{82.1} & 78.8 & 76.7 & 75.7 & \textbf{52.4} & \textbf{49.9} & \textbf{51.9} \\
    
    \midrule 
    \multirow{6}{*}{\textbf{NELL-QAC}}
    &\textsc{SQL} & 60.9 & 58.8 & 63.3 & 59.6 & \textbf{76.7} & 74.9 & 33.9 & 31.4 & 27.0 \\
    & {QTO} & 63.9 & 64.1 & 68.2 & 61.7 & 74.5 & 75.3 & -- & -- & -- \\
    &{FIT} & 63.9 & 64.6 & 68.4 & 61.7 & 73.6 & \textbf{75.7} & -- & -- & -- \\
    &{GNN-QE} & 70.4 & 69.7 & 71.2 & 72.1 & 72.2 & 74.9 & -- & -- & -- \\
    &{\textsc{UltraQuery}} & 66.3 & 65.6 & 73.2 & 71.1 & 73.2 & 73.4 & -- & -- & -- \\
    &\(\anycq\)\:\: & \textbf{76.2} & \textbf{72.3} & \textbf{79.0} & \textbf{75.4} & \textbf{76.7} & 75.3 & \textbf{57.2} & \textbf{52.6} & \textbf{58.2} \\
    \bottomrule
  \end{tabular}
  \label{tab:qac_results}
\end{table*}

\begin{table}[t]
  \centering
  % \small
  \scriptsize
  \caption{F1-scores on QAR samples, where $k$ is the number of free variables.}
  \setlength{\tabcolsep}{4pt}
  \begin{tabular}{lccccccccccccc}
    \toprule  
    \multirow{2}{*}{\textbf{Dataset}} & \multirow{2}{*}{\textbf{Model}} & \multicolumn{4}{c}{\textbf{3-hub}} & \multicolumn{4}{c}{\textbf{4-hub}} & \multicolumn{4}{c}{\textbf{5-hub}} \\
    \cmidrule(lr){3-6}
    \cmidrule(lr){7-10}
    \cmidrule(lr){11-14}
     & & \(k\!=\!1\) & \(k\!=\!2\) & \(k\!=\!3\) & \textbf{total} & \(k\!=\!1\) & \(k\!=\!2\) & \(k\!=\!3\) & \textbf{total} & \(k\!=\!1\) & \(k\!=\!2\) & \(k\!=\!3\) & \textbf{total} \\
     \midrule
     \multirow{2}{*}{\textbf{FB15k-237-QAR}}
        & \textsc{SQL} & 65.8 & 46.2 & 17.8 & 45.7 & 59.9 & 50.2 & 33.7 & 48.7 & 60.6 & 49.3 & 42.5 & 51.2 \\
        & \(\anycq\)   & \textbf{67.8} & \textbf{62.3} & \textbf{50.2} & \textbf{60.5} & \textbf{60.4} & \textbf{54.0} & \textbf{48.2} & \textbf{54.5} & \textbf{63.0} & \textbf{56.9} & \textbf{43.1} & \textbf{54.8} \\
     \midrule
     \multirow{2}{*}{\textbf{NELL-QAR}}
        & \textsc{SQL} & 63.5 & 41.3 & 24.0 & 46.7 & 60.6 & 42.1 & 32.9 & 47.7 & 52.7 & 42.5 & 27.6 & 42.8 \\
        & \(\anycq\)   & \textbf{66.7} & \textbf{55.1} & \textbf{39.1} & \textbf{55.8} & \textbf{65.1} & \textbf{57.1} & \textbf{46.5} & \textbf{57.6} & \textbf{58.7} & \textbf{51.1} & \textbf{39.6} & \textbf{51.1} \\
     \bottomrule
  \end{tabular}
  \label{tab:qar_f1scores}
\end{table}

\subsection{Main experiments results over QAC and QAR}
\paragraph{Query answer classification (QAC) experiments}
\label{sec:qac_results}
The results of evaluation on the introduced QAC benchmarks are presented in \Cref{tab:qac_results}.
As expected, GNN-QE and \textsc{UltraQuery} outperform ComplEx-based FIT and QTO, with GNN-QE displaying the best scores out of the considered baselines.
Equipped with the same NBFNet predictors, \(\anycq\) matches its performance, achieving only marginally (within \(3\%\) relative) lower F1-scores on FB15k-237-QAC, and leading by far on NELL-QAC evaluations. 
Importantly, \(\anycq\) successfully extrapolates to formulas beyond the processing power of the existing CQA approaches. On all proposed large query splits \(\anycq\) consistently outperforms the SQL baseline: SQL classifies \emph{only} easy answers accurately, mapping all the hard answers to false, and as a result falls behind \(\anycq\).

\paragraph{Query answer retrieval (QAR) experiments}
\label{sec:qar_results}
We present the QAR evaluation results across all splits of the two proposed datasets consisting of large formulas with multiple free variables in \Cref{tab:qar_f1scores}.
The performance of the SQL engine degrades as the number of free variables in the input query increases.
While a similar behavior can be witnessed for \(\anycq\) models, it progresses at a much slower rate.
Furthermore, \(\anycq\) is capable of finding non-trivial answers, even to complicated queries.
As a consequence, \(\anycq\) consistently outperforms SQL on all splits, with the biggest differences being witnessed for queries involving more than one free variable.

A further analysis (detailed in \Cref{app:extended_evaluation}) shows that \(\anycq\) does not fall behind SQL on instances admitting observable answers, remaining  within \(10\%\) relative to the classical engine on unary queries, while outperforming it on multivariate splits.
Moreover, \(\anycq\) correctly solves a~fair share of hard instances, demonstrating its ability to retrieve unobserved yet correct answers, even for large, structurally complex queries with multiple free variables.

\begin{table*}[t]
  \centering
    \vspace{-1em}
  \small
  \caption{F1-scores of \(\anycq\) models applied outside the training knowledge graph domain.}
  \begin{tabular}{cccccccc}
    \toprule    
    \multicolumn{2}{c}{\(\anycq\) \textbf{specification}} & \multicolumn{3}{c}{\textbf{FB15k-237-QAR}} & \multicolumn{3}{c}{\textbf{NELL-QAR}} \\
    \midrule
    \textbf{Predictor type} & \textbf{Training dataset} & \textbf{3-hub} & \textbf{4-hub} & \textbf{5-hub} & \textbf{3-hub} & \textbf{4-hub} & \textbf{5-hub} \\
    \midrule
    {NBFNet-based,} & FB15k-237  & 60.5 & 54.5 & 54.8 & 55.9 & 58.7 & 49.4 \\
    pre-trained on \(G\) & NELL  & 58.8 & 53.5 & 52.6 & 55.8 & 57.6 & 51.1 \\
    \midrule
    \multirow{2}{*}{perfect \(\Tilde{\pi}\) for \(\Tilde{G}\)} & FB15k-237 & 94.4 & 93.4 & 93.0 & 95.5 & 96.4 & 96.2 \\
     & NELL  & 92.2 & 90.4 & 90.2 & 94.5 & 95.7 & 94.9 \\
    \bottomrule
  \end{tabular}
  \label{tab:transferability}
\end{table*}
\begin{table*}[t]
  \centering
  
  \small
  \caption{F1-scores of \(\anycq\) model equipped with a perfect link predictor on the QAC task.}
  \begin{tabular}{cccccccccc}
    \toprule  
     \textbf{Dataset} & \textbf{2p} & \textbf{3p} & \textbf{pi} & \textbf{ip} & \textbf{inp} & \textbf{pin} & \textbf{3-hub} & \textbf{4-hub} & \textbf{5-hub} \\
    \midrule
    \textbf{FB15k-237-QAC} & 100 & 99.9 & 100 & 100 & 100 & 100 & 92.4 & 91.4 & 93.8  \\
    \textbf{NELL-QAC} & 100 & 100 & 100 & 100 & 100 & 100 & 93.0 & 89.4 & 91.3 \\
    \bottomrule
  \end{tabular}
  \label{tab:qac_perf_results}
    \vspace{-1em}
\end{table*}
\subsection{Ablation studies}
\label{sec:ablation_studies}

\paragraph{How does $\anycq$ perform outside the training domain?}\label{sec:outside_training_domain}
As mentioned in \Cref{sec:anycq_properties}, we expect the search engine to exhibit similar behavior on processed instances, regardless of the underlying knowledge graph.
We validate this claim by applying \(\anycq\) models trained on FB15k-237 or on NELL to both datasets, equipping a relevant link predictor.
The results on our QAR and QAC benchmarks are presented in \Cref{tab:transferability,tab:qac_perf_results}. respectively.
Notably, the differences between models' accuracies in QAR are marginal, confirming that the resulting search engine is versatile and not dataset-dependent.
In fact, the model trained on FB15k-237 exhibits better performance on \emph{both} datasets, further aligning with our assumption on the transferability and generalizability of \(\anycq\).

\paragraph{How does \(\anycq\) perform with a prefect link predictor?} \label{sec:perfect_lp_qar}
The \(\anycq\) framework's performance heavily depends on the underlying link prediction model, responsible for guiding the search and determining the satisfiability of generated assignments.
Hence, to assess purely the quality of our search engines, we equipped them with perfect link predictors for the test KGs, eliminating the impact of predictors' imperfections.
The results of~experiments on our QAR and QAC benchmarks are available in \Cref{tab:transferability} and \Cref{tab:qac_perf_results}, respectively. 
The performance of \(\anycq\) on all QAR splits exceeds 90\%, displaying great accuracy of the GNN-guided search engine.
We witness similar results on the large formula splits in QAC datasets, further confirming the relevance of our model and highlighting the impact of the equipped link predictor. 
Remarkably, the simple query types in QAC pose no challenge for \(\anycq\), which achieves \(100\%\) F1-score on all of them, with a single exception. 
% Similarly, for the task of QAR, \(\anycq\) with a perfect link predictor achieved over \(90\%\) F1-score, establishing the engine's excellent ability to retrieve answers to structurally complex questions.  

\section{Summary, limitations, and outlook}
In this work, we devise and study two new tasks from the query answering domain: query answer classification and query answer retrieval.
Our formulations target the challenge of classifying and generating answers to structurally complex formulas with an arbitrary number of free variables. Moreover, we introduce datasets consisting of instances beyond the processing capabilities of existing approaches, creating strong benchmarks for years to come.
To address this demanding setting, we~introduce \(\anycq\), a framework applicable 
for scoring and generating answers for large conjunctive formulas with arbitrary arity over incomplete knowledge graphs.
We demonstrate the effectiveness over our QAC and QAR benchmarks, showing that on simple samples, \(\anycq\) matches the performance of state-of-the-art CQA models, while setting challenging baselines for the large instance splits. 
One potential limitation is considering by default the input query in disjunctive normal form, converting to which may require exponentially many operations. 
We hope our work will motivate the~field of query answering to expand the scope of CQA to previously intractable cases and recognize the classification nature of the induced tasks.

\section*{Acknowledgements}
The authors would like to acknowledge the use of the University of Oxford Advanced Research Computing (ARC) facility in carrying out this work. \href{https://doi.org/10.5281/zenodo.22558}{https://doi.org/10.5281/zenodo.22558}

\bibliographystyle{unsrtnat}
\bibliography{refs}

\appendix

\section{Extended discussion of task formulations and baselines}

\subsection{Limitations of existing problem formulations} 
\label{app:limitation}

\textbf{Intractability of high-arity query evaluation.} The objective of complex query answering is to~rank all possible answers to a given logical formula.
Already for queries \(Q(x_1, x_2)\) with two free variables, this entails scoring  \(|V(G)|^2\) pairs of entities \((a_1, a_2) \in V(G)^2\), which is computationally infeasible for modern knowledge graphs \citep{fb15k237, nell} containing thousands of nodes.
As a result, most of~the~existing approaches are not designed to handle higher arity queries, either resolving to inefficient enumeration strategies \citep{fit} or approximating answers by marginal predictions.
This scalability bottleneck has already been observed by \cite{efok_dataset}, who suggested more tractable evaluation methodologies, yet again being only marginal approximations of the true performance.
Therefore, we argue that the ranking-based formulation has significantly limited the progress in query answering over formulas with multiple free variables.

\textbf{Limited structural complexity in existing benchmarks.} A related limitation lies in the structural simplicity of existing benchmarks \citep{betae}. Standard CQA literature predominantly focuses on tree-like queries, which aligns with the capabilities of most current models.
More recently, \citet{efok_dataset} introduced a dataset containing cyclic queries and queries with up to two free variables; however, the overall structures remained constrained -- featuring at most four variables and a single cycle.
We argue that addressing structurally richer queries is essential for advancing automated reasoning systems.
In real-world applications to autonomous systems, such as an AI trip planner that simultaneously books flights, accommodations, and activities while satisfying budget and availability constraints, the~underlying reasoning involves large, complex queries with multiple variables.
As AI agents become more capable, the complexity of the queries they must resolve is only expected to increase, requiring more expressive answering engines, e.g. applicable to large, cyclic, multi-variate formulas.

\textbf{Lack of probabilistic calibration in ranking-based methods.} Practical applications often demand binary decisions - answering questions like ``Is X true?” or ``What is the correct answer to Y?”, requiring models to classify candidate solutions as either 
\true\ or \false\ \citep{cqa_classification}. However, many ranking-based CQA methods do not natively support this decision-making paradigm, as they focus on ordering candidates without enforcing a meaningful threshold to distinguish valid answers from incorrect ones. While many of these models are trained using classification losses, such training does not guarantee the output scores correspond to calibrated satisfiability probabilities. In fact, several approaches rely on Noisy Contrastive Estimation \citep{lmpnn, clmpt} or apply ad hoc score-to-probability transformations \citep{qto, fit}, further weakening the reliability of predicted scores in downstream tasks.

\subsection{Limitations of existing baselines}
\label{app:baseline_discussion}

As mentioned in \Cref{sec:experiment_setup}, for the large-query splits of both QAC and QAR, none of the existing approaches are directly applicable. In particular, no method can simultaneously (1) efficiently handle \textbf{cyclic} queries and (2) provide \textbf{calibrated} probability estimates without relying on knowledge of trivial answers.

Standard methods like BetaE \citep{betae}, CQD \citep{cqd}, ConE \citep{cone}, GNN-QE \citep{gnn-qe} or QTO \citep{qto} are limited to tree-like queries. Neural approaches, such as LMPNN \citep{lmpnn} or CLMPT \citep{clmpt}, are trained using Noisy Contrastive Estimation; hence, their predictions do not meaningfully translate to desired probabilities.
Finally, FIT \citep{fit} and Q2T \citep{q2t} require transforming scores predicted by their ComplEx-based link predictors, while all known schemes (see \Cref{app:predictor_complex}) assume the set of easy answers is known, or otherwise, trivial to recover.

\section{\anycq\ details}
\label{app:anycq_details}
\subsection{Architecture}
\label{app:architecture}
\(\anycq\)'s architecture is based on the original ANYCSP \cite{anycsp} framework. The trainable components of the \(\anycq\) GNN model \(\theta\) are:
\begin{itemize}
    \item a GRU \cite{gru} cell \(\mathbf{G} : \mathbb{R}^d \times \mathbb{R}^d \rightarrow \mathbb{R}^d\) with a trainable initial state \(\mathbf{h} \in \mathbb{R}^d\)
    \item a Multi Layer Perceptron (MLP) value encoder \(\mathbf{E} : \mathbb{R}^{d+1} \rightarrow \mathbb{R}^d\)
    \item two MLPs \(\rmM_V, \rmM_R : \mathbb{R}^d \rightarrow \mathbb{R}^{4d}\) sending information between value and literal vertices
    \item three MLPs $\rmU_V, \rmU_R, \rmU_X : \mathbb{R}^d \rightarrow \mathbb{R}^d$ aggregating value, literal and variable messages
    \item an MLP \(\mathbf{O} : \mathbb{R}^d \rightarrow \mathbb{R}\) that generates logit scores for all variable nodes. 
\end{itemize}

We denote the set of neighbors of term and literal nodes by \(\gN(\cdot)\). In the case of value nodes, we~distinguish between the corresponding term node and the set of connected literal vertices, which we represent by \(\gN_R(\cdot)\).

The model starts by sampling an initial assignment \(\alpha^{(0)}\), where the value of each variable is chosen uniformly at random from \(V(G)\), and proceeds for \(T\) search steps. In step \(t\):
\begin{itemize}

\item If \(t=1\), initialize the hidden state of each value node to be \(\rvh^{(0)}(v_{z\rightarrow a}) = \rvh\).

\item Generate light edge labels under the assignment \(\alpha^{(t-1)}\) for all value-literal edges. Precisely, let \(v_{\psi_i}\) be a literal node corresponding to an atomic formula \(\psi\) and \(v_{z\rightarrow a}\) be a~connected value node. The light edge label \(L_E^{(t-1)}\left(v_{\psi_i}, v_{z\rightarrow a}; \alpha\right)\) is a binary answer to the question: ``Is $\psi$ satisfied under \(\left[\alpha^{(t-1)}\right]_{z \rightarrow a}\)?'' with respect to the equipped predictor.

\item For each value node \(v_{z\rightarrow a}\), generate its new latent state
\[\rvx^{(t)}(v_{z\rightarrow a}) = \rmE\left(\left[ \rvh^{(t-1)}(v_{z\rightarrow a}), \delta_{\alpha(x)=v} \right]\right)\]
where \([\cdot,\cdot]\) denotes concatenation and \(\delta_C=1\) if the condition \(C\) holds, and \(0\) otherwise. 
\item Derive the messages to be sent to the constraint nodes:
\[
\rvm^{(t)}(v_{z\rightarrow a}, 0), ... , \rvm^{(t)}(v_{z\rightarrow a}, 3) = \rmM_V \left(\rvx^{(t)}(v_{z\rightarrow a})\right)
\]
\item For each literal node \(v_{\psi}\), gather the messages from its value neighbors, considering \mbox{the light} and potential labels: 
\[\rvy^{(t)}(v_{\psi}) = \bigoplus_{v_{z\rightarrow a} \in \gN(v_{\psi})} \rvm^{(t)}\left(v_{z\rightarrow a}, 2\cdot P_E(v_\psi,v_{z\rightarrow a}) + L_E^{(t-1)}(v_\psi,v_{z\rightarrow a}; \alpha)\right)\]
where \(\bigoplus\) denotes element-wise \(\max\).

\item The messages to be sent to the value nodes are then evaluated as: 
\[
\rvm^{(t)}(v_\psi, 0), ... , \rvm^{(t)}(v_\psi, 3) = \rmM_R \left(\rvy^{(t)}(v_\psi)\right)\]

\item Aggregate the messages in each value node \(v_{z\rightarrow a}\): \[
\rvy^{(t)}(v_{z\rightarrow a}) = \bigoplus_{v_\psi \in \mathcal{N_R}(v_{z\rightarrow a})} \rvm^{(t)}\left(v_{z\rightarrow a}, 2\cdot P_E(v_\psi,v_{z\rightarrow a}) + L_E^{(t-1)}(v_\psi,v_{z\rightarrow a}; \alpha)\right)
\]
and integrate them with current hidden state: \[\rvz^{(t)}(v_{z\rightarrow a}) = \rmU_V \left(\rvx^{(t)}(v_{z\rightarrow a}) + \rvy^{(t)}(v_{z\rightarrow a})\right) + \rvx^{(t)}(v_{z\rightarrow a})\]

\item For each term node \(v_z\), aggregate the states of the corresponding value nodes:
\[
\rvz^{(t)}(v_z) = \rmU_X \left( \bigoplus_{v_{z\rightarrow a}\in\mathcal{N}(v_z)} \rvz^{(t)}(v_{z\rightarrow a})\right)
\]

\item For each value node \(v_{z\rightarrow a}\), update its hidden state as: 
\[
\rvh^{(t)}(v_{z\rightarrow a}) = \rmG\left(\rvh^{(t-1)}(v_{z\rightarrow a}), \rvz^{(t)}(v_{z\rightarrow a}) + \rvz^{t}(v_z)\right)
\]

\item Generate logits and apply softmax within each domain:
\[
\begin{aligned}
    \rvo^{(t)}_{z\rightarrow a} &= \text{clip}\left(\rmO\left( \rvh^{(t)}(v_{z\rightarrow a}) \right) - \max_{a\in\gD(z)} \rmO\left( \rvh^{(t)}(v_{z\rightarrow a}) \right), [-100, 0]\right)\\
\mu^{(t)}(v_{z\rightarrow a}) &= \frac{\exp \rvo^{(t)}_{z\rightarrow a}}{\sum_{a'\in \mathcal{D}(z)} \exp \rvo^{(t)}_{z\rightarrow a'}}
\end{aligned}
\]

\item Sample the next assignment $\alpha^{(t)}$, selecting the next value independently for each variable \(x\), with probabilities $\sP\left(\alpha^{(t)}(x)=a\right) = \mu^{(t)}(v_{x\rightarrow a})$ for all $a\in\gD(x)$.
\end{itemize}

Note that the suggested methodology for evaluating probabilities \(\sP\left(\alpha^{(t)}(x) = a\right)\) is approximately equivalent to applying softmax directly on \(\rmO\left(\rvh^{(t)}(v_{x\rightarrow a})\right)\). However, applying this augmentation, we are guaranteed that for any variable \(x\) and a relevant value \(a\in\gD(x)\):
\[
\sP\left(\alpha^{(t)}(x) = a\right) =  \frac{\exp \rvo^{(t)}_{x\rightarrow a}}{\sum_{a'\in \mathcal{D}(x)} \exp \rvo^{(t)}_{x\rightarrow a'}} \geq \frac{e^{-100}}{|\gD(x)|} \geq \frac{1}{e^{100}|V(G)|}.
\]

\subsection{Training methodology}
\label{app:anycq_training}
Suppose we are given a training query \(Q(x) = \exists \vec{y}. \Phi(x, \vec{y})\).
We run \(\Theta\) on \(\exists {x} . Q(x)\) for~\(T_\text{train}\) search steps, recovering the assignments \(\alpha^{(0)}, ..., \alpha^{(T_\text{train})}\) and the intermediate value probability distributions:
\[
\mu^{(1)} = \left\{\mu_z^{(1)} | z\in\{\vec{x},\vec{y}\}\right\}, \:\dots\: ,\: \mu^{(T_\text{train})} = \left\{\mu_z^{(T_\text{train})} | z\in\{\vec{x},\vec{y}\}\right\}
\]
The reward \(R^{(t)}\) for step \(1 \leq t \leq T\) is calculated as the difference between the score for~assignment \(\alpha^{(t)}\) and the best assignment visited so far:
\[
R^{(t)} = \max \left( 0, S^{(t)} - \max_{t'<t} S^{(t')} \right)
\]
where \(S^{(t)} = S_{\pi_\text{train}}\left(\Phi(\alpha^{(t)}(x)/x, \alpha^{(t)}(\vec{y})/\vec{y})\right)\).
Additionally, the transition probability
\[
P^{(t)} = \sP\left(\alpha^{(t)} | \mu^{(t)} \right) = \prod_{z\in\{\vec{x},\vec{y}\}} \mu_z^{(t)}\left(\alpha^{(t)}(z)\right)
\]
represents the chance of drawing assignment \(\alpha^{(t)}\) at step \(t\), given distributions \(\left\{\mu_z^{(t)} | z\in\{\vec{x},\vec{y}\}\right\}\).
The corresponding REINFORCE's training loss is evaluated as a weighted sum of rewards generated during \(T_\text{train}\) search steps and the model weights are then updated using the gradient descend equation:
\[
\theta \leftarrow \theta - \alpha \:\cdot\:\nabla_\theta \left( -\sum_{i=0}^{T_\text{train}-1} \gamma^{i} \left(\left(\log P^{(t)}\right) \cdot \sum^{T_\text{train}}_{t=i+1}\left(\gamma^{t-i-1} R^{(t)}\right) \right) \right)
\]
where \(\gamma \in (0,1]\) is a discount factor and  \(\alpha \in \mathbb{R}\) is the learning rate.

For the training data, we use the training splits of the existing FB15k-237 and NELL CQA datasets \cite{betae}, consisting of queries of types: `1p', `2p', `3p', `2i', `3i', `2in', `3in', `pin', `inp' (see \Cref{tab:small_types} for the corresponding first-order logic formulas).
Hence, during training, \(\anycq\) witnesses queries with projections, intersections and negations, learning principles of this logical structures.
However, all of these queries mention at most 3 free variables, remaining limited in size. 

\begin{table}[t]
    \centering
    \caption{Simple query types}
    \begin{tabular}{cc}
    \toprule
         \textbf{Split} & \textbf{Formula} \\
         \midrule
         1p & \(Q(x_1) = \mathsf{r}_1(x, c_1)\) \\
         \midrule
         2p & \(Q(x_1) = \exists y_1 . \mathsf{r}_1(x_1, y_1) \land \mathsf{r}_2(y_1, c_1)\) \\
         \midrule
         3p & \(Q(x_1) = \exists y_1, y_2. \mathsf{r}_1(x_1, y_1) \land \mathsf{r}_2(y_1, y_2)\land \mathsf{r}_3(y_2, c_1)\) \\
         \midrule
         2i & \(Q(x_1) = \mathsf{r}_1(x, c_1) \land \mathsf{r}_2(x, c_2)\) \\
         \midrule
         3i & \(Q(x_1) = \mathsf{r}_1(x, c_1) \land \mathsf{r}_2(x, c_2) \land \mathsf{r}_3(x, c_3)\) \\
         \midrule
         pi & \(Q(x_1) = \exists y_1 . \mathsf{r}_1(x_1, y_1) \land \mathsf{r}_2(y_1, c_1) \land \mathsf{r}_3(x_1, c_2)\) \\
         \midrule
         ip & \(Q(x_1) = \exists y_1 . \mathsf{r}_1(x_1, y_1) \land \mathsf{r}_2(y_1, a_1) \land \mathsf{r}_3(y_1, a_2)\) \\
         \midrule
         2i & \(Q(x_1) = \mathsf{r}_1(x, c_1) \land \neg \mathsf{r}_2(x, c_2)\) \\
         \midrule
         3i & \(Q(x_1) = \mathsf{r}_1(x, c_1) \land \mathsf{r}_2(x, c_2) \land \neg\mathsf{r}_3(x, c_3)\) \\
         \midrule
         inp & \(Q(x_1) = \exists y_1 . \mathsf{r}_1(x_1, y_1) \land \mathsf{r}_2(y_1, c_1) \land \neg \mathsf{r}_3(y_1, c_2)\) \\
         \midrule
         pin & \(Q(x_1) = \exists y_1 . \mathsf{r}_1(x_1, y_1) \land \mathsf{r}_2(y_1, c_1) \land \neg \mathsf{r}_3(x_1, c_2)\) \\
         \bottomrule
    \end{tabular}
    \label{tab:small_types}
\end{table}

\subsection{Hyperparameters and implementation}
\label{app:hyperparameters}

\textbf{Architecture.} We choose the hidden embedding size \(d=128\) in the \(\anycq\) architecture for all experiments. 
All MLPs used in our model consist of two fully connected layers with ReLU \cite{agarap2018deep} activation function. The intermediate dimension of the hidden layer is chosen to be 128.

\textbf{Training.}
The REINFORCE \cite{reinforce} discount factor \(\lambda\) is set to \(0.75\) for both datasets, following the best configurations in ANYCSP experiments.
During training, we run our models for~\(T_\text{train} = 15\) steps.
The batch size is set to \(4\) for FB15k-237 and \(1\) for NELL, due to the GPU memory constraints.
All models are trained with an Adam \cite{adam} optimizer with learning rate \(5\cdot10^{-6}\) on a single NVIDIA Tesla V100 SXM2 with 32GB VRAM.
We let the training run for 4 days, which translates to 500,000 batches on FB15k-237 and 200,000 batches for NELL, and choose the final model for testing.

\textbf{Inference.}
To run all experiments, we use an Intel Xenon Gold 6326 processor with 128GB RAM, and an NVIDIA A10 graphics card with 24GB VRAM.   

\subsection{Trained \anycq\: versus random search}
\label{sec:anycq_vs_random}

To clarify the contribution of the training procedure in \anycq, we conducted an additional ablation comparing the fully trained model against an untrained random-search baseline.
In this case, each variable assignment is chosen independently and uniformly at random during the search process.
The~resulting F1 scores, evaluated on QAC benchmarks, are presented in Table~\ref{tab:qac_random_vs_trained}.
Unsurprisingly, the performance drops drastically when training is removed. In particular, the random search fails completely on large-query splits, where the model must correctly assign values to over eight variables.

That said, we would like to acknowledge the potential of “simpler” strategies, such as greedy search or hill-climbing, as promising future directions.
While such algorithms were not applicable in the ranking-based formulation of prior work (which required full entity rankings), our classification-based setup now makes them feasible. Exploring these strategies could yield competitive baselines for complex query answering; however, we consider this beyond the scope of the current study.

\begin{table}[t]
\small
\centering
\caption{F1-scores of trained and randomly initialzed models on QAC datasets.}
\begin{tabular}{l c c c c c c c c c c}
\toprule
\textbf{Dataset} & \textbf{Model} & \textbf{2p} & \textbf{3p} & \textbf{pi} & \textbf{ip} & \textbf{inp} & \textbf{pin} & \textbf{3-hub} & \textbf{4-hub} & \textbf{5-hub} \\
\midrule
\multirow{2}{*}{FB15k-237-QAC}
& random & 3.3 & 0.0 & 4.9 & 2.9 & 4.4 & 3.5 & 0.0 & 0.0 & 0.0 \\
& \anycq & 75.8 & 71.3 & 82.1 & 78.8 & 76.7 & 75.7 & 52.4 & 49.9 & 51.9 \\
\midrule
\multirow{2}{*}{NELL-QAC}
& random & 2.8 & 0.0 & 3.3 & 2.4 & 3.1 & 3.1 & 0.0 & 0.0 & 0.0 \\
& \anycq & 76.2 & 72.3 & 79.0 & 75.4 & 76.7 & 75.3 & 57.2 & 52.6 & 58.2 \\
\bottomrule
\end{tabular}
\label{tab:qac_random_vs_trained}
\end{table}

\subsection{Scope of formulas}
\label{app:scope_of_formulas}
Importantly, our method is not limited to conjunctive formulas. 
Suppose we are given a Boolean formula \(\varphi = \exists \vec{y} . \Psi(\vec{y})\) where \(\Psi(\vec{y})\) is quantifier-free and in disjunctive normal form (DNF), so that \(\Psi(\vec{y}) = C_1 \lor ... \lor C_n\) where each \(C_i\) is a~conjunction of~literals.
Then:
\[
\varphi \equiv (\exists \vec{y} . C_1) \lor ... \lor (\exists \vec{y} . C_n)
\]
which can be processed by \(\anycq\) by independently solving each \((\exists \vec{y} . C_i)\) and aggregating the results.
Moreover, the ability of our model to handle higher arity relations enables efficient satisfiability evaluation for~existential formulas in the conjunctive normal form.
Let \(\psi = \exists \vec{y} . \left(D_1 \land ... \land D_n\right)\) where each \(D_i\) is a disjunction of literals.
Consider \(D_i = l_{i,1} \lor ... \lor l_{i,m}\) and let \(\vec{z}_i = \var(D_i)\).
We~can~view the~disjunctive~clause \(D_i\) as a single relation \(D_i(\vec{z}_i)\) evaluating to
\[
S_{\pi,G}(D_i(\alpha(\vec{z}_i)/\vec{z}_i)) = \max_{j} S_{\pi,G}(l_{i,j}(\alpha(\var(l_{i,j}))/\var(l_{i,j})))
\] 
Under this transformation, \(\psi = \exists \vec{y} . \left(D_1(\vec{z}_1) \land ... \land D_n(\vec{z}_n)\right)\) becomes a conjunctive query, hence processable by \(\anycq\). Up to our knowledge, we present the first query answering approach efficiently scoring arbitrary CNF Boolean queries over incomplete knowledge graphs.  

\subsection{Expressivity}
\label{app:expressivity}
Standard graph neural networks are known to have limited expressive power \cite{xu18}, e.g. MPNNs cannot produce different outputs for graphs not distinguishable by~the~Weisfeiler-Lehman algorithm \cite{leman1968reduction}.
We argue that \(\anycq\) does not suffer from this limitation.
It has been noticed that including randomness in GNN models increases their expressiveness \cite{gnn-rni}. 
In our case, for any Boolean conjunctive query \(Q=\exists\vec{y} \Phi(\vec{y})\) over a knowledge~graph~\(G\) and~a~relevant link predictor \(\pi\), for any assignment \(\alpha : \{\vec{y}\} \to V(G)\), there is a non-zero probability of \(\alpha\) being selected at some point of the search (see \Cref{app:architecture}).
Hence, any \(\anycq\) model has a chance of correctly predicting \(S_{\pi, G}(Q)\), making it fully expressive for the tasks of QAC and QAR.  

\subsection{Fuzzy logic}
\label{app:fuzzy_logic}

Fuzzy logic has been widely adopted in the CQA literature as a way to evaluate the satisfiability of~logical formulas in a continuous, differentiable manner. It underpins several prominent methods, including CQD~\citep{cqd}, GNN-QE~\citep{gnn-qe}, and QTO~\citep{qto}, due to its modularity and interpretability. However, especially when applied to large and structurally complex queries, fuzzy logic introduces several limitations that should be taken into account.

\textbf{Score vanishing. } Consider a conjunction of \(10\) literals, each scored at \(0.9\) by the link predictor. When using the product \textit{t}-norm, the formula score becomes \(0.9^{10} \approx 0.35\), despite all individual facts being highly probable.
This effect becomes more pronounced in long formulas, leading to overly conservative judgments.
To mitigate this, we adopt the G\"odel \textit{t}-norm (min operator), which in the same scenario would return a~more stable score of \(0.9\). Additionally, using the G\"odel t-norm with a \(0.5\) threshold is equivalent to applying propositional logic over binarized literal scores, making it well-suited for our classification-based objectives.

\textbf{Gradient instability in supervised learning. }
As discussed by~\citet{van2022analyzing}, another issue with fuzzy logic arises in~differentiable learning settings, where gradients must propagate through the query structure and the fuzzy connectives.
This can lead to vanishing or unstable gradients, especially for large or cyclic queries.
In our case, however, this problem is largely avoided: \(\anycq\) is trained using reinforcement learning, where the fuzzy logic score is used as a scalar reward signal and not differentiated through.
During training, we apply REINFORCE, which treats the Boolean score as an external reward, and during inference, fuzzy logic is only used to rank complete assignments. As~such, our framework sidesteps the gradient-related challenges described in~\citet{van2022analyzing}, while retaining the benefits of fuzzy logic for scoring.

\subsection{Edge labels}
\label{app:edge_labels}

To effectively navigate the space of variable assignments, our framework augments the computational graph \(G_{Q,\alpha}\) with edge labels that encode information from the link predictor \(\pi\). These edge labels play a critical role in guiding the search process by helping the model answer two fundamental questions:
\begin{itemize}
    \item Which assignments to variables are worth considering? (exploration)
    \item How should the current assignment be changed to satisfy more literals? (exploitation).
\end{itemize}

To this end, we define two types of edge labels on the graph edges connecting literal vertices \(v_{\psi_i}\) with value vertices \(v_{e\rightarrow a}\): potential edge (PE) labels and light edge (LE) labels. PE labels are used to identify whether a~particular substitution could lead to a satisfying assignment and are independent of the current state. They support \emph{exploration} by indicating globally promising directions in the search space and can be seen as a way to constrain the search to regions of high potential. In contrast, LE labels are assignment-dependent and indicate whether a local change - modifying a single variable’s value, would make a particular literal true. They enable exploitation by directing the model toward refinements of the current assignment that increase the satisfiability of the formula. We describe each type of label in detail in the following subsections.

\subsubsection{LE Labels: Guiding Local Improvements}
\label{app:le_labels}

Light edge (LE) labels were originally introduced in the ANYCSP \citep{anycsp} framework as the primary mechanism for guiding discrete search. In the context of query answering, their purpose is to identify marginal changes to the current variable assignment that are likely to increase the number of satisfied literals in the query. That is, given an assignment \(\alpha\), LE labels help determine which single-variable substitutions are most promising for improving the current solution. This makes them particularly useful during local exploitation, where the goal is to refine an existing candidate assignment rather than explore the full space.

\textbf{Formal definition. } Let \(Q = \exists \vec{y} . \Phi(\vec{y})\) be an existential Boolean conjunctive query, and let \(\psi_i\) be a literal in~\(Q\), mentioning terms \(\vec{z}\). Let \(\alpha\) be the current assignment to the variables of \(Q\), and let \(e\in\vec{z}\) be a term in~\(\psi_i\). For a candidate entity \(a\in\mathcal{D}(e)\) (recall that \(\mathcal{D}(e) = V(G)\) for variables and \(\mathcal{D}(e) = \{e\}\) for constants), the~LE label on the edge between the literal vertex \(v_{\psi_i}\) and the value vertex \(v_{e\rightarrow a}\) is defined as follows:
\[
L_E (v_{\psi_i}, v_{z,a}; \alpha) = 
\begin{cases}
    1 & \text{ if  } S_{\pi,G} (\psi_i(\alpha_{z\rightarrow a}(\vec{z})/\vec{z})) \geq 0.5  \\
    0 & \text{ otherwise }
\end{cases}
\]
This definition reflects whether updating the current assignment \(\alpha\) by changing only the value of~\(e\)~to~\(a\) (keeping all other variable assignments fixed) is sufficient to make the literal \(\psi_i\) true.

\textbf{Example. } Suppose the query is:
\[
Q = \exists y_1, y_2 . r(a, y_1) \land s(y_1, y_2)
\]
with the current assignment \(\alpha = \{y_1 \rightarrow a_1, y_2 \rightarrow a_2\}\), and we focus on the literal \(\psi_2 = s(y_1, y_2)\). Let's consider a marginal update to the variable \(y_2\), and let \(a_2' \in \mathcal{D}(y_2)\). To determine the LE label \(L_E(v_{\psi_2}, v_{y_2 \rightarrow a_2'}; \alpha)\), we check whether \(s(a_1, a_2')\) holds in the (predicted) KG \(\Tilde{G}\).  If it does, then this local update would satisfy \(\psi_2\), and the label is set to 1. Otherwise, the label is 0.
This allows the~model to reason about whether changing \(y_2\) to \(a_2'\) would improve the current assignment in terms of satisfying more of the query structure.

\subsubsection{PE Labels: Prioritizing Promising Assignments}
\label{app:pe_labels}

Potential edge (PE) labels are introduced in this work as an extension to the ANYCSP framework, specifically to address the challenges posed by the large domain sizes in modern knowledge graphs.
While LE labels guide the refinement of a given assignment, PE labels serve a complementary purpose: they help identify which candidate variable assignments are worth considering in the first place.
In~other words, PE labels support exploration by informing the model which edges in the computational graph represent substitutions that are likely to participate in some satisfying assignment, independent of the current state.

\textbf{Formal definition. } Formally, let \(Q = \exists \vec{y}. \Phi(\vec{y})\) be a conjunctive Boolean query, let \(\psi_i \in \Phi\) be a literal mentioning terms \(\vec{z}\), and let \(e\in\vec{z}\). Then, for every \(a\in\mathcal{D}(e)\), the PE label on the edge between \(v_{\psi_i}\) and \(v_{e\rightarrow a}\) is defined as follows:
\[
P_E(v_{\psi_i}, v_{z,a}) = 
\begin{cases}
    1 & \text{ if  } \exists \alpha . \left( \alpha(e) = a \land S_{\pi,G} (\psi_i(\alpha(\vec{z})/\vec{z})) \geq 0.5 \right) \\
    0 & \text{ otherwise }
\end{cases}
\]
Intuitively, the label is set to 1 if there exists any full assignment to the variables of \(\psi_i\) such that \(\psi_i\) becomes true when \(e\) is set to \(a\).
Importantly, this is evaluated without reference to the current partial assignment \(\alpha\), making PE labels suitable for filtering the search space early in the computation.

\textbf{Example. } Consider the same example query as before:
\[
Q = \exists y_1, y_2 . r(a, y_1) \land s(y_1, y_2)
\]
and the literal \(\psi_2 = s(y_1, y_2)\). Let \(a_2 \in \mathcal{D}(y_2)\) be a viable assignment to \(y_2\).
To evaluate the PE label \(P_E(v_{\psi_2}, v_{y_2 \rightarrow a_2})\), we check whether \(\exists y_1 . s(y_1, a_2)\) is satisfied, i.e. whether there exists some \(a_1\) such that the literal \(s(a_1, a_2)\) is true, according to the link predictor \(\pi\).
If such \(a_1\) exists, we set the label to~1, and otherwise -- to 0.
This allows the GNN to prioritize reasoning about value assignments that could plausibly contribute to satisfying the query, rather than wasting capacity on highly unlikely candidates.

\paragraph{Importance of PE labels.} 
We empirically validate the significance of this modification on the proposed QAR benchmark.
To this end, we train an \(\anycq\) model from scratch, disabling the signal from PE labels by setting all of them to \(0\) throughout the training and inference.  
The results, shown in \Cref{tab:PEimpact}, demonstrate that without access to PE labels, \(\anycq\) fails to generalize to queries of large size and is unable to produce a correct answer, even for a single sample.

\begin{table}[t]
  \centering
  \caption{F1-scores of \(\anycq\) models with and without PE labels.}
  \label{tab:PEimpact}
\setlength{\tabcolsep}{5pt}
  \begin{tabular}{ccccccc}
    \toprule    
     & \multicolumn{3}{c}{\textbf{FB15k-237-QAR}} & \multicolumn{3}{c}{\textbf{NELL-QAR}} \\
    \midrule
     \textbf{PE labels} & \textbf{3-hub} & \textbf{4-hub} & \textbf{5-hub} & \textbf{3-hub} & \textbf{4-hub} & \textbf{5-hub} \\
    \midrule
     \checkmark  & 56.3 & 52.7 & 54.1  & 51.4 & 53.0 & 48.4 \\
     \(\times\) & 0.0 & 0.0 & 0.0 & 0.0 & 0.0 & 0.0  \\
    \bottomrule
  \end{tabular}
  
\end{table}

\paragraph{PE label generation.}
Given the critical role of this modification in our framework, it is essential to address the efficient generation of PE labels.
In this work, we pre-compute PE labels for both datasets, aligning them precisely with the definitions, with respect to the selected test link predictors.
However, this process can become computationally expensive, potentially requiring hours, and becoming highly inefficient, particularly in scenarios where the link predictor frequently changes, e.g. during validation.
To mitigate this inefficiency, we propose alternative methodologies to approximate true PE labels, enabling faster cold-start inference.

Our main alternative bases on the closed world assumption (CWA) \cite{open_world_assumption}, which restricts the set of entities that should be considered for prediction of unobserved facts.
Formally, let \(G\) be an observable knowledge graph and let \(\Tilde{G}\) be its completion.
Then, for any \(r\in R(G)\) and any \(a,b\in V(G)\):
\[
\Tilde{G} \models r(a,b) \implies \exists b' \in V(G) \:.\: G \models r(a,b') 
\]
\[
\Tilde{G} \models r(a,b) \implies \exists a' \in V(G) \:.\: G \models r(a',b) 
\]
With this assumption, the set of pairs for which \(\Tilde{G}\models r(a,b)\) holds becomes limited.
Indeed, \(a\) needs to be a head of an observable relation \(r(a,b')\) and analogously, \(b\) needs to be a tail of an observable \(r(a',b)\).
Therefore, the induced approximation of PE labels:
\[
\Hat{P}_E (v_{r(x,y)}, v_{x,a}) =
\begin{cases}
    1 & \text{ if  } \exists b'\in \mathcal{D}(y) . G \models r(a,b') \\
    0 & \text{ otherwise }
\end{cases}
\]
\[
\Hat{P}_E (v_{r(x,y)}, v_{y,b}) =
\begin{cases}
    1 & \text{ if  } \exists a'\in \mathcal{D}(x) . G \models r(a',b) \\
    0 & \text{ otherwise }
\end{cases}
\]
can be efficiently derived in time \(O(|E(G)|)\).
We use this modification during the validation process to avoid the necessity of computing the precise PE labels.

An alternative approach, not explored in this work, involves incorporating domain-specific information about the underlying knowledge graph.
For instance, if the relation in a given query is \textit{fatherOf}, both entities are likely to be humans.
By labeling all entities in \(V(G)\) with relevant tags, such information could be extracted, and objects classified as `people' could be assigned a corresponding PE label of~1.
While we prioritize generalizability and do not pursue this direction, we recognize its potential, particularly for sparse knowledge graphs where CWA-derived PE labels may be too restrictive.

\paragraph{PE labels versus domain restriction.}
An alternative to relying on an additional set of labels to prevent the search from accessing unreasonable assignments could be restricting the domains \(\mathcal{D}(y)\) of~the~considered variables.
In the current formulation, each variable \(y\) mentioned in the input Boolean query \(Q\) is assigned a domain \(\mathcal{D}(y) = V(G)\).
Reducing the considered domains can significantly shrink the computational graph, leading to faster computation.
Such a solution would be specifically beneficial when operating on large knowledge graphs, and even essential for applications to milion-scale KGs, such as Wikidata-5M \citep{wikidata5m}. While this approach improves inference efficiency, improper application can render correct answers unreachable due to excessively restrictive domain reductions. Consequently, we leave further exploration of this direction for future work.

\clearpage

\section{Theorems and proofs}
\label{app:proofs}
\begin{theorem}
\label{thm:completeness}
Let \(Q = \exists\vec{y} . \Phi(\vec{y})\) be a conjunctive Boolean query and let \(\Theta\) be any $\anycq$ model equipped with a predictor \(\pi\). For any execution of \(\Theta\) on \(Q\), running for \(T\) steps:
\[
\sP\left(\Theta(Q | G, \pi) = S_{\pi,G}(Q)\right) \rightarrow 1 \qquad \text{as}\qquad T \rightarrow \infty 
\]
\end{theorem}
\begin{proof}
Let \(\Theta\) be an \(\anycq\) model equipped with a predictor \(\pi\) for a knowledge graph \(G\). Let \(Q=\exists \vec{y}.\Phi(\vec{y})\) be a conjunctive Boolean query with \(\vec{y} = y_1, ..., y_k\). Let
\[
\alpha_{\max} = \argmax_{\alpha : \vec{y} \to V(G)} S_{\pi,G}(\Phi(\alpha(\vec{y})/\vec{y}))
\]
so that
\[
S_{\pi,G}(Q) = S_{\pi,G}(\Phi(\alpha_{\max}(\vec{y})/\vec{y}))
\]
Consider an execution of \(\Theta\), running for \(T\) steps, and let \(\alpha^{(0)}, ..., \alpha^{(T)}\) be the generated assignments. Then,
\[
\begin{aligned}
\sP\left(\Theta(Q | G, \pi) \neq S_{\pi,G}(Q)\right) &=  \sP \left( S_{\pi,G}(Q) \neq  S_{\pi,G}\left(\Phi\left(\alpha^{(t)}(\vec{y})/\vec{y}\right)\right) \text{ for all } 0\leq t\leq T
\right) \\
&\leq \sP \left( S_{\pi,G}(Q) \neq  S_{\pi,G}\left(\Phi\left(\alpha^{(t)}(\vec{y})/\vec{y}\right)\right) \text{ for all } 1\leq t\leq T
\right) \\
&\leq \sP \left( \alpha^{(t)} \neq \alpha_{\max} \text{ for all } 1\leq t\leq T
\right)
\end{aligned}
\]
By the remark at the end on \Cref{app:architecture}:
\[
\sP\left(\alpha^{(t)}(y) = a\right) \geq \frac{1}{e^{100} |V(G)|} \qquad \forall 1 \leq t \leq T \:\:\forall a\in V(G)\:\:\forall y \in \vec{y} 
\]
In particular:
\[
\sP\left(\alpha^{(t)}(y) = \alpha_{\max}(y)\right) \geq \frac{1}{e^{100} |V(G)|} \qquad \forall 1 \leq t \leq T \:\: \forall y \in \vec{y}
\]
so as the value for each variable in \(\alpha^{(t)}\) is sampled independently:
\[
\sP\left(\alpha^{(t)} = \alpha_{\max}\right) \geq \left( \frac{1}{e^{100} |V(G)|} \right) ^ k
\]
Therefore:
\[
\begin{aligned}
\sP\left(\Theta(Q | G, \pi) \neq S_{\pi,G}(Q)\right) &\leq \sP \left( \alpha^{(t)} \neq \alpha_{\max} \text{ for all } 1\leq t\leq T
\right) \\
& \leq \left( 1 - \left( \frac{1}{e^{100} |V(G)|} \right) ^ k \right)^T \\
\end{aligned}
\]
which tends to 0 as \(T \rightarrow \infty\).
\end{proof}

\clearpage

\begin{proposition}[Scores of a Perfect Link Predictor]
\label{app:perf_pred_scores}
Let \(Q\) be a quantifier-free Boolean formula over an observable knowledge graph \(G\). Then, the score of \(Q\) w.r.t. the perfect link predictor \(\Tilde{\pi}\) for the completion \(\Tilde{G}\) of \(G\) satisfies:
    \[
    S_{\Tilde{\pi},G}(Q) = \begin{cases}
        0 & \text{if } \Tilde{G} \not\models Q \\
        1 & \text{if } \Tilde{G} \models Q
    \end{cases}
    \]
\end{proposition}
\begin{proof}
    The claim follows from the structural induction on the formula \(Q\).
    For the base case, suppose that \(Q\) is an atomic formula \(r(a,b)\). The result follows trivially from the definition of a perfect link predictor \(\Tilde{\pi}\). Assume the claim holds for boolean formulas \(Q, Q'\). Then:
    \[
    \begin{aligned}
        S_{\Tilde{\pi},G}(\neg Q) = 1 - S_{\Tilde{\pi},G}(Q) = \begin{cases}
        1 & \text{if } \Tilde{G} \not\models Q \\
        0 & \text{if } \Tilde{G} \models Q
    \end{cases}
    = \begin{cases}
        1 & \text{if } \Tilde{G} \models \neg Q \\
        0 & \text{if } \Tilde{G} \not\models \neg Q
    \end{cases}
    \end{aligned}
    \]
    For \((Q \land Q')\), note that \(\Tilde{G} \models (Q\land Q') \iff \left( (\Tilde{G}\models Q) \land (\Tilde{G} \models  Q')\right) \) and hence
    \[
    \begin{aligned}
    S_{\Tilde{\pi},G}(Q \land  Q') = \min\left(S_{\Tilde{\pi},G}(Q), S_{\Tilde{\pi},G}( Q')\right) &= \begin{cases}
        1 & \text{if } S_{\Tilde{\pi},G}(Q) = S_{\Tilde{\pi},G}( Q')=1 \\
        0 & \text{otherwise }
    \end{cases} \\
    &= \begin{cases}
        1 & \text{if } \Tilde{G} \models Q \land \Tilde{G} \models  Q' \\
        0 & \text{otherwise }
    \end{cases} \\
    &= \begin{cases}
        1 & \text{if } \Tilde{G} \models (Q \land  Q') \\
        0 & \text{if } \Tilde{G} \not\models (Q \land  Q')
    \end{cases}
    \end{aligned}
    \]
    Similarly, for \((Q\lor Q')\), since \(\Tilde{G} \models (Q\lor Q') \iff \left( (\Tilde{G}\models Q) \lor (\Tilde{G} \models  Q')\right) \), we can deduce:
    \[
    \begin{aligned}
    S_{\Tilde{\pi},G}(Q \lor  Q') = \max\left(S_{\Tilde{\pi},G}(Q), S_{\Tilde{\pi},G}( Q')\right) &= \begin{cases}
        0 & \text{if } S_{\Tilde{\pi},G}(Q) = S_{\Tilde{\pi},G}( Q')=0 \\
        1 & \text{otherwise }
    \end{cases} \\
    &= \begin{cases}
        0 & \text{if } \Tilde{G} \not\models Q \land \Tilde{G} \not\models  Q' \\
        1 & \text{otherwise }
    \end{cases} \\
    &= \begin{cases}
        0 & \text{if } \Tilde{G} \not\models (Q \lor  Q') \\
        1 & \text{if } \Tilde{G} \models (Q \lor  Q')
    \end{cases}
    \end{aligned}
    \]
    which completes the inductive step.
\end{proof}

\begin{theorem}
    Let \(Q = \exists\vec{y} . Q(\vec{y})\) be a conjunctive Boolean query over an unobservable knowledge graph \(\Tilde{G}\) and let \(\Theta\) be any $\anycq$ model equipped with a perfect link predictor \(\Tilde{\pi}\) for \(\Tilde{G}\). If \(\Theta(Q | G, \Tilde{\pi}) > 0.5\), then \(\Tilde{G} \models Q\) .
    \label{thm:soundness}
\end{theorem}
\begin{proof}
    Consider the setup as in the theorem statement and suppose \(\Theta(Q|G,\Tilde{\pi}) > 0.5\). Then, there exists an assignment \(\alpha : \vec{y} \to V(G)\) (found at some search step) such that
    \[S_{\Tilde{\pi},G}(\Phi(\alpha(\vec{y})/\vec{y})) = \Theta(Q|G,\Tilde{\pi}) > 0.5\]
    By \Cref{app:perf_pred_scores}, this implies:
    \[
    S_{\Tilde{\pi},G}(\Phi(\alpha(\vec{y})/\vec{y})) = 1 \qquad \text{ and } \qquad \Tilde{G} \models \Phi(\alpha(\vec{y})/\vec{y}) 
    \]
    Hence, \(\Tilde{G} \models \exists\vec{y} . \Phi(\vec{y}) = Q\).
\end{proof}

\vcqgenprocessfig

\begin{figure}[ht]
    \centering
    \includegraphics[scale=0.234]{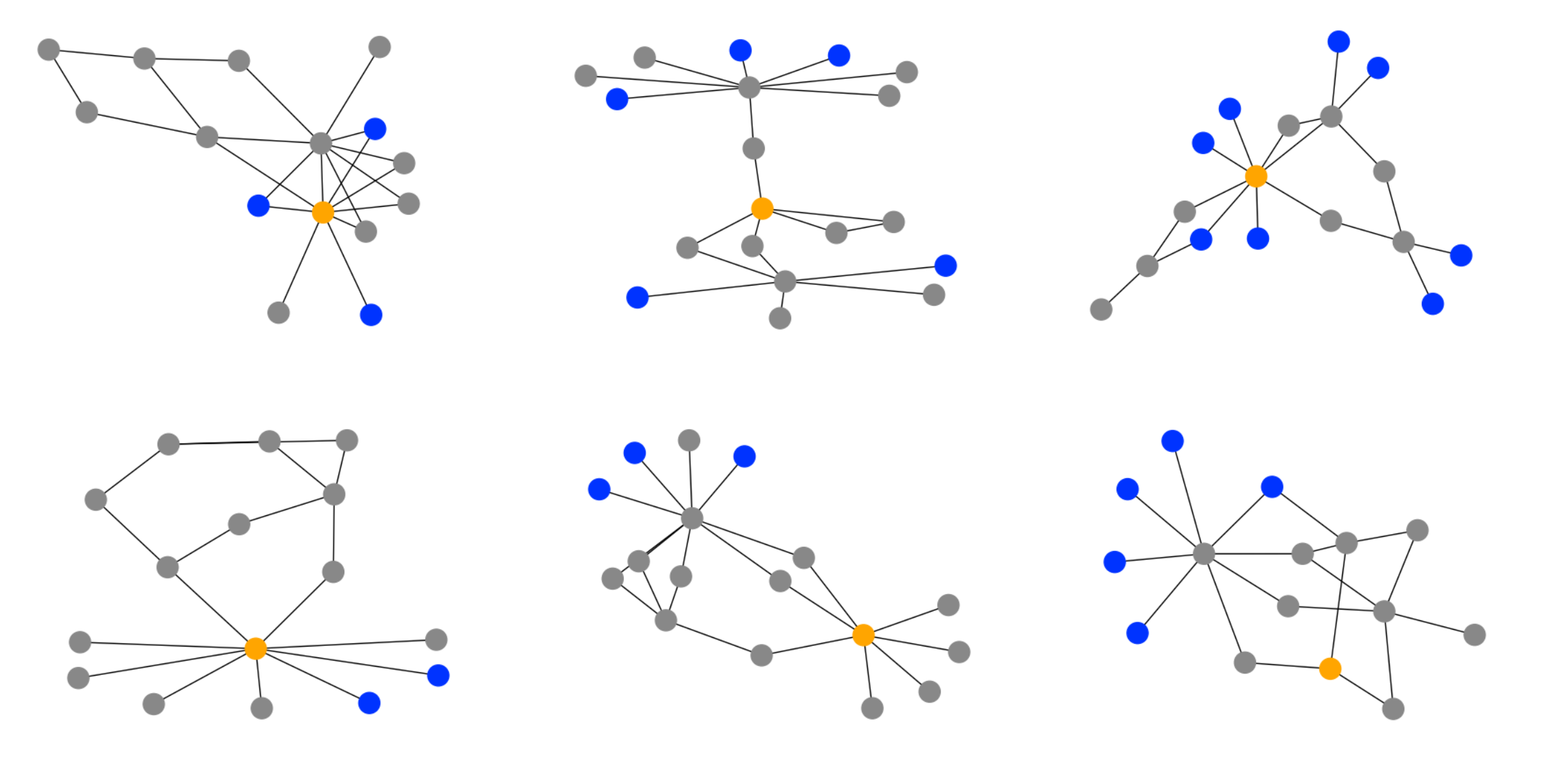}
    \includegraphics[scale=0.165]{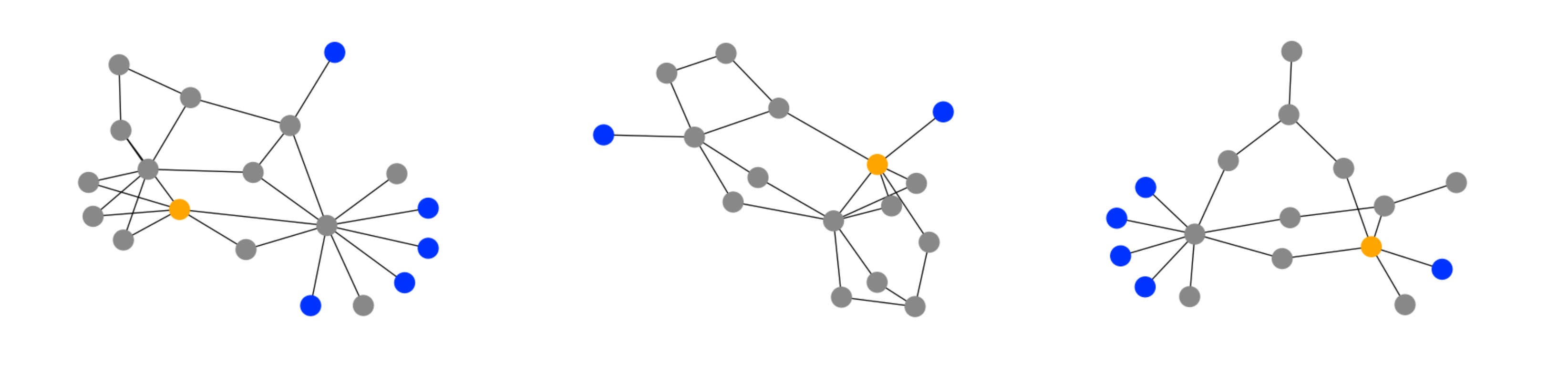}
    \caption{Examples of undirected query graphs of formulas from the FB15k-237-QAR `3-hub' split. Blue nodes represent constant terms, while grey - to the existentially quantified variables. The orange node corresponds to the free variable.}
\end{figure}
\clearpage

\section{Dataset constructions}
\label{app:dataset_generation}

Benchmark datasets in the existing query answering literature, FB15k-237 \cite{fb15k237} and NELL \cite{nell}, comprise formulas with simple structures, thereby impeding the comprehensive evaluation and advancement of methodologies and algorithms.
We address this gap by creating new validation and test datasets on top of well-established benchmarks, consisting of queries with complexity exceeding the processing power of known approaches.
In particular, we increase the number of~variables mentioned in the considered formulas from 3 to between 12 and 20, while imposing structural difficulty by sampling query graphs with multiple cycles, long-distance reasoning steps and multi-way conjunctions.

\subsection{Base large query generation}

Each of the considered datasets: FB15k-237 and NELL, provides three knowledge graphs \(G_\text{train}, G_\text{val}\), \( G_\text{test}\), for training, validation and testing, respectively, satisfying \mbox{\(E(G_\text{train}) \subset E(G_\text{val}) \subset E(G_\text{test}) \)}.
During validation, \(G_\text{train}\) is treated as the observable graph \(G\), while \(G_\text{val}\) as~its~completion \(\Tilde{G}\). 
Similarly, for testing, \(G=G_\text{val}\) and \(\Tilde{G} = G_\text{test}\).

We begin the dataset generation by sampling base formulas, to be later converted into instances for the QAC and QAR benchmarks.
During sampling, we use four hyperparameters: \(n_{\text{hub}}, n_{\text{min}}, p_{\text{const}} \text{ and } p_{\text{out}}\), whose different values contribute to creating different benchmark splits.
The process is visualized in \Cref{fig:vcq_gen_process}. 
A single base query is sampled as follows:
\begin{enumerate}
    \item A vertex \(v \in V(G)\) is sampled uniformly at random from \(V(G)\).
    \item Let \(\mathcal{N}_i(v)\) be the set of nodes whose distance from \(v\) in \(\Tilde{G}\) is at most \(i\). Without repetitions, sample \(n_{\text{hub}}\) `hub' vertices from \(\mathcal{N}_2(v)\) and call their set \(P\). If \(|\mathcal{N}_2(v)| < n_{\text{hub}}\), return to step 1.
    \item Consider the union of 1-hop neighborhoods of the `hub' vertices: \(D = \bigcup_{w\in P\cup\{v\}} \mathcal{N}_1(w)\).
    \item If \(w\in D\) is a leaf in the restriction \(\Tilde{G}_D\) of \(\Tilde{G}\) to D, remove it from \(D\) with probability \(p_{\text{out}}\).
    \item Sample a set \(D'\) of \(n_{\min}\) vertices from \(D\), such that the restriction of \(\Tilde{G}\) to \(D'\cup P \cup \{v\}\) is~a~connected subgraph. Let \(P' = D'\cup P \cup \{v\}\). If the restriction \(\Tilde{G}_{P'}\) of \(\Tilde{G}\) to \(P'\) is~a~subgraph of the observable graph \(G\), return to step 1.
    \item For each node \(w\) in \(D'\) independently, choose it to be portrayed by a constant term with~probability \(\frac{p_{\text{const}}}{d^2_{P'}(w)}\), where \(d_{P'}(w)\) is the degree of \(w\) in restriction of \(\Tilde{G}\) to \(P'\).
    \item The restriction \(\Tilde{G}_{P'}\) of \(\Tilde{G}\) to \(P'\) is then converted into the corresponding conjunctive formula, by transforming each edge \(r(w_1, w_2) \in E(\Tilde{G}_{P'})\) into a literal \(r(w_1, w_2)\). The~vertex~\(v\) is then replaced by the single free variable \(x_1\) and all nodes that were not chosen to be constant, are realized by distinct existentially quantified variables.
\end{enumerate}

For formulas sampled from FB15k-237, we choose \(n_{\min}=15\), while for NELL instances, \(n_{\min}=12\), due to the sparsity of~the~knowledge graph.
We consider three different choices of the parameters \(n_{\text{hub}}, p_{\text{const}} \text{ and } p_{\text{out}}\), resulting in three distinct splits, namely ``3-hub'', ``4-hub'' and ``5-hub'', and sample 1000 formulas of each type.
Using an SQL engine \citep{duckdb}, we then solve these queries with respect to both observable and unobservable knowledge graphs, discarding those with no hard answers.
The parameter values for each split are presented in \Cref{tab:gen_params}.

\begin{table}[t]
    \centering
    \caption{Hyperparameters for the generated dataset splits.}
    \label{tab:gen_params}
    \begin{tabular}{cccc}
    \toprule
         \textbf{Split} & $n_{\text{hub}}$ & \(p_{\text{const}}\) & $p_{\text{out}}$\\
    \midrule
         3-hub & 2 & 0.6 & 0.95 \\ 
         4-hub & 3 & 0.8 & 0.97 \\ 
         5-hub & 4 & 1.0 & 0.99 \\
    \bottomrule
    \end{tabular}
\end{table}

\subsection{Query answer classification datasets}

We propose two benchmarks for query answer classification: FB15k-237-QAC and~\mbox{NELL-QAC}.
Instances in each dataset are stored in a unified form:
\[
(Q(x), C_Q, W_Q)
\]
where \(Q(x)\) is the input formula and \(C_Q, W_Q\) are subsets of \(V(G)\) with \(|C_Q| = |W_Q|\) such that:
\[
\Tilde{G} \models Q(a/x) \quad \forall a \in C_Q \qquad\qquad\text{and}\qquad\qquad
\Tilde{G} \not\models Q(b/x) \quad \forall b \in W_Q
\]

Each of our QAC benchmarks includes 9 splits, which can be broadly divided into two parts. Their statistics are more broadly described in \Cref{tab:qac_data_properties}.
\(\%\)easy, \(\%\)hard, and \(\%\)neg represent the proportions of easy answers, hard answers, and incorrect proposals in each split, respectively.

In the first part of our benchmarking, we utilize samples from existing CQA datasets, focusing exclusively on formulas that include projections. This choice is crucial, as grounding the free variables in non-projection queries (e.g., `2i', `3i', `2u', `2in', `3in') reduces the task to a set of independent link prediction problems, which do not meaningfully test reasoning capabilities beyond atomic fact retrieval.
Similarly, disjunctive queries (e.g., `up') can be decomposed into independent subqueries under the QAC setting, introducing little additional complexity and offering limited insight into a model's reasoning abilities.

We instead select a representative subset of six query types: `2p', `3p', `ip', `pi', `inp', and `pin', spanning key logical constructs such as projection, conjunction, and negation. This selection allows for both robust evaluation and continuity with prior work, enabling meaningful comparison with classical and neural CQA baselines under the classification-based objective. For each query type, we sample 500 queries to ensure a balanced and reliable evaluation.

For the main components of FB15k-237-QAC and NELL-QAC, we convert large base queries into QAC instances, reducing the size of each split to 300 queries.
These samples are characterized by significant structural complexity, presenting a substantial challenge for both existing and future query answering methods.

In both cases, the size \(|C_Q|=|W_Q|\) is chosen as \(\text{\textit{clip}}(|\{a\in V(G) : \Tilde{G}\models Q(a/x)\}|, 5, 10)\). \(W_Q\) is then sampled uniformly from the set of incorrect groundings for \(Q(x)\), while \(C_Q\) is drawn from the set of answers to \(Q(x)\), assigning non-trivial answers twice higher probability than the easy ones.

\begin{table}[t]
\centering
  \caption{Statistics of introduced QAC datasets.}
  \centering
  \begin{tabular}{lccccccccc}
    \toprule  
     & \textbf{2p} & \textbf{3p} & \textbf{pi} & \textbf{ip} & \textbf{inp} & \textbf{pin} & \textbf{3-hub} & \textbf{4-hub} & \textbf{5-hub} \\
    \midrule
    &  \multicolumn{9}{c}{\textbf{FB15k-237-QAC}} \\
    \midrule
    \textbf{\#queries} & 500 & 500 & 500 & 500 & 500 & 500 & 300 & 300 & 300  \\
    \textbf{\#answers} & 9818 & 9828 & 9632 & 9358 & 9808 & 9898 & 2036 & 1988 & 2028 \\
    \textbf{\(\%\)easy} & 26.5\% & 24.0\% & 27.5\% & 28.7\% & 35.8\% & 32.9\% & 18.7\% & 16.6\% & 17.0\%  \\
    \textbf{\(\%\)hard} & 23.5\% & 26.0\% & 22.5\% & 21.3\% & 14.2\% & 17.1\% & 31.3\% & 33.4\% & 33.0\%  \\
    \textbf{\(\%\)neg} & 50.0\% & 50.0\% & 50.0\% & 50.0\% & 50.0\% & 50.0\% & 50.0\% & 50.0\% & 50.0\%  \\
    \midrule
    & \multicolumn{9}{c}{\textbf{NELL-QAC}} \\
    \midrule
    \textbf{\#queries}  & 500 & 500 & 500 & 500 & 500 & 500 & 300 & 300 & 300 \\
    \textbf{\#answers}  & 9708  & 9702 & 9478 & 9694 & 9698 & 9888 & 2174 & 2186 & 1922 \\
    \textbf{\(\%\)easy} & 23.6\% & 22.6\% & 25.2\% & 23.6\% & 35.8\%& 32.8\% & 15.9\% & 14.4\% & 13.7\% \\
    \textbf{\(\%\)hard} & 26.4\% & 27.4\% & 24.8\% & 26.4\% & 14.2\% & 17.2\% & 34.1\% & 35.6\% & 36.3\% \\
    \textbf{\(\%\)neg}  & 50.0\% & 50.0\% & 50.0\% & 50.0\% & 50.0\% & 50.0\% & 50.0\% & 50.0\% & 50.0\% \\
    \bottomrule
  \end{tabular}
  \label{tab:qac_data_properties}
\end{table}
\begin{table}[t]
  \centering
  \caption{Statistics of introduced QAR datasets.}
  \begin{tabular}{lcccccc}
    \toprule  
    &  \multicolumn{3}{c}{\textbf{FB15k-237-QAR}} & \multicolumn{3}{c}{\textbf{NELL-QAR}}  \\
    \midrule
     & \textbf{3-hub} & \textbf{4-hub} & \textbf{5-hub} & \textbf{3-hub} & \textbf{4-hub} & \textbf{5-hub} \\
    \midrule
    \textbf{\#queries} & 1200 & 1200 & 1200 & 1000 & 1000 & 1000 \\
    \textbf{\#trivial} & 565 & 537 & 586 & 387 & 416 & 417 \\
    \textbf{\#free=1} & 400 & 400 & 400 & 400 & 400 & 400 \\
    \textbf{\#free=2} & 400 & 400 & 400 & 300 & 300 & 300 \\
    \textbf{\#free=3} & 400 & 400 & 400 & 300 & 300 & 300  \\
    \bottomrule
  \end{tabular}
  \label{tab:qar_data_properties}
\end{table}

\subsection{Query answer retrieval datasets}

Most samples in CQA benchmarks yield answers within the observable knowledge graph \(G\). 
Due to their simplicity, these instances are trivial for query answer retrieval, as classical solvers can efficiently derive the correct answers.
Consequently, we do not include such small queries in our FB15k-237-QAR and NELL-QAR datasets. Instead, we focus on addressing the limitations of current benchmarks by including more complex queries involving multiple free variables.

For the single free variable case, we select 400 base queries from each split. To generate formulas of arity 2, we randomly remove the quantification over one of the existentially quantified variables.
The~resulting query is then solved using an SQL engine, leveraging information from the initial answer set to optimize computation.
An analogous methodology is applied to extend the arity 2 formulas to instances with 3 free variables.
Statistics of the generated test splits are available in \Cref{tab:qar_data_properties}. \textbf{\#trivial} is the number of samples admitting a trivial answer, and \textbf{\#free=k} - arity \(k\) formulas.

\subsection{Evaluation protocol}
\label{app:dataset_evaluation_methodology}

\textbf{Query answer classification.} We use the F1-score as the metric to measure the performance on the task of query answer classification.
The reported F1-scores (\Cref{tab:qac_results}) are an average of F1-scores for single instances \((Q(x), C_Q, W_Q)\) taken over the whole dataset.
Formally, letting \(\mathcal{D}\) be the~considered dataset and denoting by \(A(\theta, Q)\) the set of entities from \(C_Q \cup W_Q\) marked by the~model~\(\theta\) as correct answers to \(Q(x)\), we report:
\[
F1_{\text{QAC}}(\theta) = \frac{1}{|\mathcal{D}|} \sum_{(Q(x), C_Q, W_Q) \in \mathcal{D}} \frac{2|A(\theta, Q) \cap C_Q|}{2|A(\theta, Q) \cap C_Q| + |A(\theta, Q) \backslash C_Q| + |W_Q \cap A(\theta, Q)|}
\]

\textbf{Query answer retrieval.} We adapt the F1-score metric to the task of QAR. In particular, we count a~positive outcome (i.e. solution prediction) as correct if and only if it is a true answer to the query.
Given a~model~\(\theta\), let \(\text{Rec}(\theta)\) be the proportion of \emph{correctly answered} positive instances in the dataset, while \(\text{Prec}(\theta)\) be the ratio of \emph{correctly answered} positive instances among the queries for which \(\theta\) predicted a solution. We then report:
\[
F1_\text{QAR}(\theta) = \frac{2}{\frac{1}{\text{Prec}(\theta)}+\frac{1}{\text{Rec}(\theta)}}
\]

\section{Link predictors}
\label{sec:used_complex_lp}
As mentioned in \Cref{sec:computational_graph}, we incorporate a link predictor into our architecture, to address the problem of deducing facts not presented in the observable knowledge graph. 
We consider three different model types from the existing CQA literature: transductive knowledge graph embedding method ComplEx \citep{trouillon2016complex} used in QTO and FIT, inductive (on nodes) method NBFNet \citep{nbfnet} employed by GNN-QE, and inductive (on nodes and relations) knowledge graph foundation model \textsc{Ultra} \citep{galkin2023ultra}, lying at the heart of \textsc{UltraQuery}.

\subsection{ComplEx}
\label{app:predictor_complex}
Recall that a ComplEx model \(\chi\) assigns each entity \(e \in V(G)\) and each relation \(r \in R(G)\), a~\(d_{\chi}\)-dimensional complex-valued vector \(v_e, w_r \in \mathbb{C}^{d_{\chi}}\).
We choose the hidden dimension of $d_{\chi}=1000$ for all experiments.
For each triple \((r,a,b) \in R(G) \times V(G) \times V(G)\), the score of the entities \(a,b\) being in relation \(r\) is derived as:
\begin{equation*}
    \chi (r,a,b)
    =   \Re e\left( \langle v_a, w_r, \overline{v}_b \rangle\right) 
    =  \Re e\left( \sum_{i=1}^{d_{\chi}} (v_a)_i (w_r)_i \overline{(v_b)_i}\right)
\end{equation*}

\paragraph{Training.}
\label{sec:lp_training}
For training, we follow the relation prediction methodology, presented in \cite{relation_prediction}, evaluating the~loss as a sum over all known facts \(r(a,b) \in E(G)\) of three cross-entropy losses, marginalizing the head, the relation and the tail:
\begin{equation*}
    \mathcal{L}_r(\chi) = -\sum_{r(a,b)\in E(G)}
    \Big( \log(p_{\chi,\tau}(a | r,b)) + \log(p_{\chi,\tau}(b | a,r)) + \lambda_{rel}\log(p_{\chi,\tau}(r | a,b))\Big) + \mathcal{L}_{reg}
\end{equation*}
where \(\mathcal{L}_{reg}\) is a nuclear 3-norm \cite{nuclear_norm} regularization term and the marginal probabilities are evaluated as:
\[
\mathcal{L}_{reg} = \sum_{i=1}^{d_{\chi}} \left(2\cdot\!\sum_{a\in V(G)} |(v_a)_i|^3 + \sum_{r\in R(G)} |(w_r)_i|^3\right)   
\]
\begin{equation*}
    p_{\chi,\tau}(a | r,b) = \frac{\exp(\tau\cdot\chi(r,a,b))}{\sum_{a'\in V(G)} \exp(\tau\cdot\chi(r,a',b))}
\end{equation*}
\[
    p_{\chi,\tau}(b | a,r) = \frac{\exp(\tau\cdot\chi(r,a,b))}{\sum_{b'\in V(G)} \exp(\tau\cdot\chi(r,a,b'))} 
\]
\[\\
    p_{\chi,\tau}(r | a,b) = \frac{\exp(\tau\cdot\chi(r,a,b))}{\sum_{r'\in R(G)} \exp(\tau\cdot\chi(r',a,b)))}\]
where \(\tau\) is a factor controlling the temperature of the applied softmax function. During training, we set \(\tau=1\).
For each dataset, the model is trained using the AdaGrad \cite{adagrad} optimizer with a learning rate $0.1$ for 500 epochs, and the checkpoint maximizing validation accuracy is chosen for testing.

\paragraph{Conversion to the probability domain.}
To match the definition of a link predictor from \Cref{sec:preliminaries}, the uncalibrated scores \(\chi(r,a,b)\) assigned by the ComplEx model \(\chi\) need to be converted into probabilities \(\rho_\mathbb{C}(r,a,b) = \sP(r(a,b)\in E(\Tilde{G}) | \chi) \).
We follow the ideas used in QTO \citep{qto} and FIT \cite{fit}, and set them as proportional to the marginal probabilities \(p_{\chi,\tau}(b|a,r)\). By definition, \(p_{\chi,\tau}(\cdot|a,r)\) defines a distribution over \(V(G)\):
\[
\sum_{b\in V(G)}  p_{\chi,\tau}(b|a,r) = 1
\]
Therefore, to match the objective:
\[
\sum_{b\in V(G)} \sP (r(a,b) \in E(\Tilde{G}) | \chi) = \left|\left\{b\in V(G) : r(a,b) \in E(\Tilde{G})\right\}\right|
\]
we multiply the marginal probabilities by a scaling factor \(Q_{a,r}\), specific to the pair \((a,r)\):
\[
\rho_\mathbb{C}(r,a,b) = \sP (r(a,b) \in E(\Tilde{G}) | \chi) = Q_{a,r}\cdot p_{\chi,\tau}(b|a,r)
\]
We consider two scaling schemes: \(Q^{\text{QTO}}_{a,r}\) introduced in QTO, and \(Q_{a,r}^{\text{FIT}}\) described by FIT.
Both methods base on the cardinality of the set \(E_{a,r} = \{b\in V(G) : r(a,b) \in E(G)\}\)  of trivial answers to the query \(Q(x) = r(a,x)\):
\[
\begin{aligned}
    Q^{\text{QTO}}_{a,r} &= \left| E_{a,r} \right| \\
    Q^{\text{FIT}}_{a,r} &= \frac{\left| E_{a,r} \right|}{\sum_{b \in E_{a,r}} p_{\chi,\tau}(b|a,r)} 
\end{aligned}
\]
During validation, we search for the best values for \(\tau\) among \([0.5,1,2,5,10,20]\) on each validation query type. We notice that \(\tau = 20\) performs best in all experiments.
The resulting link predictors \(\rho_\mathbb{C}^\text{FIT}\) and \(\rho_\mathbb{C}^\text{QTO}\), after augmenting them with links from the observable graphs as described below, are then plugged into the respective neuro-symbolic frameworks for QTO and FIT evaluations on small-query QAC splits. 
For experiments with \(\anycq\) equipped with ComplEx-based predictors, we~use the FIT, as it proved more accurate during validation.

\subsection{NBFNet}
\label{app:predictor_nbfnet}
As the second studied predictor, we consider Neural Bellman-Ford Network \citep{nbfnet}, constituting the main processing unit in GNN-QE \citep{gnn-qe}.
For the \(\anycq\) experiments, we reuse the NBFNet checkpoints obtained from training GNN-QE over the considered datasets.
We follow the configurations from the original repository -- models are trained for 10 epochs, processing 48,000 instances per epoch for the FB15k-237 training, and 24,000 samples per epoch for NELL, with Adam \citep{adam} optimizer with learning rate 0.005.
We validate 0.25 to be the optimal threshold for binarizing GNN-QE predictions, and apply it for the small-query QAC experiments.
When testing \(\anycq\) with the underlying NBFNet models, we first binarize the output of the NBFNet \(\nu\):
\[
\rho_\nu(r,a,b) = \begin{cases}
    1 & \text{if } \nu\big(r(a,b)\big) \geq t \\
    0 & \text{if } \nu\big(r(a,b)\big) < t
\end{cases}
\]
After validation, we set \(t = 0.5\) for the small-query FB15k-237-QAC splits, \(t=0.4\) for the small-query NELL-QAC splits and \(t = 0.6\)
for all large-query evaluations.

\subsection{\textsc{Ultra}}
\label{app:predictor_ultra}

Finally, to test \(\anycq\)'s ability of inductive link prediction over unseen relation, we consider \textsc{Ultra}~\citep{galkin2023ultra}, a prominent knowledge graph foundation models, as the third studied predictor for zero-shot inference.
For the \(\anycq\) experiments, we directly apply the 3g checkpoints from the original \textsc{Ultra} repository, which are pre-trained on FB15k-237~\citep{fb15k237}, WN18RR~\citep{Dettmers2018FB}, and CoDEx Medium~\citep{safavi-koutra-2020-codex} for 10 epochs with 80,000 steps per epoch, with AdamW~\citep{loshchilov2019decoupledweightdecayregularization} using learning rate of 0.0005.
Similarly to the methodology applied to NBFNet, we binarize the output of the \textsc{Ultra} model \(\upsilon\) when equipping it to \(\anycq\). 
In this case, following validation, we choose \(t=0.4\) as the best threshold for small query QAC experiments, and a higher \(t=0.9\) for formulas in large QAC and QAR splits.

As an additional baseline, we compare another state-of-the-art CQA method over unseen relation \textsc{UltraQuery}~\citep{galkin2024ultraquery}, which also utilizes \textsc{Ultra} as its link predictor.
We validate that 0.2 is the best answer classification threshold for \textsc{UltraQuery} checkpoints provided ni the original repository, trained only on the CQA benchmark based on FB15k-237. 
We highlight that the results of \textsc{UltraQuery} on NELL are hence \textit{zero-shot inference}, since NELL is not in the pretraining dataset of the evaluated checkpoints. 

\subsection{Incorporating the observable knowledge graph}
To account for the knowledge available in the observable graph \(G\), we augment all considered link predictors \(\rho\), setting \(\rho(r,a,b) = 1\) if \(r(a,b)\in E(G)\).
To distinguish between known and predicted connections, we clip the~predictor's estimations to the range \([0,0.9999]\). Combining all these steps together, given a predictor \(\rho\), in our experiments we use:
\[
\pi(r,a,b) = \begin{cases}
    1 & \text{if } r(a,b) \in E(G) \\
    \min\big(\rho(r, a, b), 0.9999\big) & \text{otherwise}
\end{cases}
\]
This methodology is applied for all \(\anycq\) experiments, to each of \(\rho_\mathbb{C}, \rho_\nu\) and \(\rho_\upsilon\), obtaining the final \(\pi_\mathbb{C}, \pi_\nu\) and \(\pi_\upsilon\), directly used for~ComplEx-based, NBFNet-based and \textsc{Ultra}-based evaluations, respectively.

\begin{table*}[t]
  \centering
  \small
    \caption{Average F1-scores of \anycq\: on the query answer classification task.}
  \begin{tabular}{ccccccccccc}
    \toprule  
     \textbf{Dataset} & \textbf{Predictor} & \textbf{2p} & \textbf{3p} & \textbf{pi} & \textbf{ip} & \textbf{inp} & \textbf{pin} & \textbf{3-hub} & \textbf{4-hub} & \textbf{5-hub} \\
    \midrule
    \multirow{3}{*}{\textbf{FB15k-237-QAC}}
    & {ComplEx} & 66.9 & 63.1 & 70.7 & 67.6 & \textbf{78.4} & 75.2 & 39.5 & 32.3 & 36.1 \\
    & {NBFNet} & \textbf{75.8} & \textbf{71.3} & \textbf{82.1} & \textbf{78.}8 & 76.7 & \textbf{75.7} & \textbf{52.4} & \textbf{49.9} & \textbf{51.9} \\
    & {\textsc{Ultra}} & 70.4 & 56.2 & 77.3 & 70.6 & 72.4 & 73.0 & 32.6 & 26.9 & 29.1 \\
    \midrule
    \multirow{3}{*}{\textbf{NELL-QAC}}
    & {ComplEx} & 63.8 & 64.0 & 68.2 & 61.7 & 74.8 & 75.0 & 39.1 & 40.0 & 34.9 \\
    & {NBFNet} & \textbf{76.2} & \textbf{72.3} & 79.0 & 75.4 & \textbf{76.7} & \textbf{75.3} & \textbf{57.2} & \textbf{52.6} & \textbf{58.2} \\
    & {\textsc{Ultra}} & 76.0 & 23.0 & \textbf{81.2} & \textbf{76.3} & 70.8 & 74.0 & 33.2 & 30.8 & 25.5 \\
    \bottomrule
  \end{tabular}
  \label{tab:qac_predictors_results}
\end{table*}

\begin{table}[t]
  \centering
  % \small
  \scriptsize
  \caption{F1-scores of \(\anycq\) equipped with different predictors on the QAR datasets.
    }
  \setlength{\tabcolsep}{3pt}
  \begin{tabular}{cccccccccccccc}
    \toprule   
    \multirow{2}{*}{\textbf{Dataset}} & \multirow{2}{*}{\textbf{Predictor}}&  \multicolumn{4}{c}{\textbf{3-hub}} & \multicolumn{4}{c}{\textbf{4-hub}} & \multicolumn{4}{c}{\textbf{5-hub}} \\
    \cmidrule(lr){3-6}
    \cmidrule(lr){7-10}
    \cmidrule(lr){11-14}
     & & \(k\!=\!1\) & \(k\!=\!2\) & \(k\!=\!3\)& \textbf{total}& \(k\!=\!1\) & \(k\!=\!2\) & \(k\!=\!3\)& \textbf{total}& \(k\!=\!1\) & \(k\!=\!2\) & \(k\!=\!3\)& \textbf{total} \\
    \midrule
    \multirow{4}{*}{\textbf{FB15k-237-QAR}}
    &\textsc{SQL} & 65.8 & 46.2 & 17.8 & 45.7 & 59.9 & 50.2 & 33.7 & 48.7 & 60.6 & 49.3 & 42.5 & 51.2 \\
    & ComplEx &67.3 & 56.3 & 43.4 & 56.3 & 57.7 & \textbf{54.4} & 45.6 & 52.7 & 62.8 & 54.3 & \textbf{44.1} & 54.1 \\
    & NBFNet & \textbf{67.8} & \textbf{62.3} & \textbf{50.2} & \textbf{60.5} & \textbf{60.4} & 54.0 & \textbf{48.2} & \textbf{54.5} & \textbf{63.0} & \textbf{56.9} & 43.1 & \textbf{54.8} \\
    & \textsc{Ultra} & 65.3 & 57.1 & 44.1 &56.0 & 57.1 & 52.4 & 42.2 & 50.8 & 59.4 & 54.3 & 41.3 & 52.0 \\
    \midrule
\multirow{4}{*}{\textbf{NELL-QAR}}
    & \textsc{SQL} & 63.5 & 41.3 & 24.0 & 46.7 & 60.6 & 42.1 & 32.9 & 47.7 & 52.7 & 42.5 & 27.6 & 42.8 \\
    & ComplEx & 62.8 & 50.0 & 34.6 & 51.4 & 61.7 & 52.1 & 40.7 & 53.0 & 55.1 & 50.0 & 36.5 & 48.4 \\
    & NBFNet & \textbf{66.7} & \textbf{55.1} & \textbf{39.1} & \textbf{55.8} & \textbf{65.1} & \textbf{57.1} & \textbf{46.5} & \textbf{57.6} & \textbf{58.7} & \textbf{51.1} & \textbf{39.6} & \textbf{51.1} \\
    &\textsc{Ultra} & 57.4 & 44.6 & 31.5 & 46.5 & 56.4 & 43.8 & 35.2 & 46.7 & 49.1 & 40.4 & 31.5 & 41.5 \\
    \bottomrule
  \end{tabular}
  \label{tab:qar_different_predictors}
\end{table}

\subsection{Combination with \anycq}
\label{app:choice_of_predictor}
As mentioned in \Cref{sec:anycq_properties}, the \(\anycq\) framework can be equipped with any link predictor capable of predicting relations over the studied knowledge graph.
For this reason, to ensure that our choice matches the most accurate setup, we validate the performance of the predictors described in previous subsections, and test \(\anycq\) combined with ComplEx-based predictor with FIT scaling (\Cref{app:predictor_complex}), NBFNet (\Cref{app:predictor_nbfnet}) and \textsc{Ultra} (\Cref{app:predictor_ultra}).
Following validation, we choose NBFNet to be equipped for \(\anycq\) evaluations in all main experiments (\Cref{tab:qac_results} and \Cref{tab:qar_f1scores}).

As an additional ablation study, we generate the test results of the remaining combinations. The results on the QAC task are presented in \Cref{tab:qac_predictors_results}, while the scores on QAR benchmarks are shown in \Cref{tab:qar_different_predictors}.

For the small-query QAC splits, we observe that NBFNet and \textsc{Ultra} strongly outperform ComplEx on positive formulas (``2p'', ``3p'', ``ip'', ``pi''), while struggling more with queries involving negations (``inp'', ``pin''). The observed drop in \textsc{Ultra}’s results on the ``3p'' split is due to the model predicting too many links to be true, assigning high confidence scores to a large fraction of candidate triples. This behavior effectively expands the plausible search space, making it exponentially harder for the RL policy to converge to correct assignments. When too many entities are deemed likely, both local (\emph{LE}) and global (\emph{PE}) supervision signals become less informative, degrading the quality of guidance.

Overall, the performance of \anycq\ depends strongly on the characteristics of the underlying link predictor. As shown in Table~\ref{tab:qar_different_predictors}, predictors trained with classification objectives (e.g., NBFNet, \textsc{Ultra}) generally outperform embedding-based models such as ComplEx. However, in practice, predictors with higher precision and more selective scoring distributions provide stronger search guidance for \anycq. These results suggest that, for complex query answering, prioritizing precision over recall in link predictor training yields more stable and effective search behavior.

For large-query classification, NBFNet demonstrates the most robust performance, achieving F1 scores exceeding 50\% on nearly all splits, while the remaining predictors consistently stay below 40\%. The evaluations on the QAR benchmarks confirm this trend: NBFNet consistently outperforms both ComplEx and \textsc{Ultra}, achieving the best results on most FB15k-237 and all NELL splits.

\textbf{Transferability.} Interestingly, we again point out that the used \textsc{Ultra} model has not been trained on the NELL dataset.
Regardless, it manages to match the performance of the ComplEx-based predictor on NELL-QAC and NELL-QAR benchmarks.
Combining this observation with our ablation of transferability of \(\anycq\) models between datasets (\Cref{sec:ablation_studies}, \Cref{tab:transferability}), we can assume that similar results would be achieved when running the evaluation with \(\anycq\) model trained on FB15k-237.
Such framework would then answer queries over NELL in a true zero-shot, fully inductive setting.
In future work, we look forward to exploring combinations of \(\anycq\) search engines trained over broad, multi-dataset data, with fully inductive link predictors (like \textsc{Ultra}), to achieve foundation models capable of answering arbitrary queries over arbitrary, even unseen, knowledge graphs.

\section{Extended evaluation over QAR}
\label{app:extended_evaluation}

\begin{table}[t]
  \centering
  % \small
  \scriptsize
  \caption{Average F1-scores on easy and hard QAR samples, where $k$ is the number of free variables.}
  \setlength{\tabcolsep}{4pt}
  \begin{tabular}{llcccccccccccc}
    \toprule  
    \multirow{2}{*}{\textbf{Dataset}} & \multirow{2}{*}{\textbf{Model}} & \multicolumn{4}{c}{\textbf{3-hub}} & \multicolumn{4}{c}{\textbf{4-hub}} & \multicolumn{4}{c}{\textbf{5-hub}} \\
    \cmidrule(lr){3-6}
    \cmidrule(lr){7-10}
    \cmidrule(lr){11-14}
     & & \(k\!=\!1\) & \(k\!=\!2\) & \(k\!=\!3\) & \textbf{total} & \(k\!=\!1\) & \(k\!=\!2\) & \(k\!=\!3\) & \textbf{total} & \(k\!=\!1\) & \(k\!=\!2\) & \(k\!=\!3\) & \textbf{total} \\
     \midrule
     \multicolumn{14}{c}{\textit{Easy instances}} \\
     \midrule
     \multirow{2}{*}{\textbf{FB15k-237-QAR}}
        & \textsc{SQL} & \textbf{94.0} & 75.9 & 41.8 & 77.1 & \textbf{93.2} & \textbf{83.5} & 68.9 & 83.7 & \textbf{92.1} & 77.7 & \textbf{76.1} & \textbf{82.7} \\
        & \(\anycq\)   & 90.8 & \textbf{91.4} & \textbf{83.4} & \textbf{89.2} & 88.3 & 83.1 & \textbf{83.8} & \textbf{85.2} & 85.0 & \textbf{80.2} & 68.7 & 78.8 \\
     \midrule
     \multirow{2}{*}{\textbf{NELL-QAR}}
        & \textsc{SQL} & \textbf{97.2} & 81.7 & 71.3 & 88.6 & \textbf{94.8} & 81.7 & 77.6 & 87.6 & 94.1 & 88.0 & 78.0 & 89.0 \\
        & \(\anycq\)   & 96.3 & \textbf{94.4} & \textbf{92.7} & \textbf{95.1} & 94.5 & \textbf{94.0} & \textbf{88.6} & \textbf{93.1} & \textbf{94.4} & \textbf{91.6} & \textbf{89.5} & \textbf{92.5} \\
     \midrule
     \multicolumn{14}{c}{\textit{Hard instances}} \\
     \midrule
     \multirow{2}{*}{\textbf{FB15k-237-QAR}}
        & \textsc{SQL} & 0 & 0 & 0 & 0 & 0 & 0 & 0 & 0 & 0 & 0 & 0 & 0 \\
        & \(\anycq\)   & \textbf{20.9} & \textbf{14.5} & \textbf{20.0} & \textbf{18.5} & \textbf{16.2} & \textbf{13.1} & \textbf{12.2} & \textbf{13.8} & \textbf{28.8} & \textbf{19.5} & \textbf{14.8} & \textbf{20.9} \\
     \midrule
     \multirow{2}{*}{\textbf{NELL-QAR}}
        & \textsc{SQL} & 0 & 0 & 0 & 0 & 0 & 0 & 0 & 0 & 0 & 0 & 0 & 0 \\
        & \(\anycq\)   & \textbf{13.1} & \textbf{13.1} & \textbf{7.7} & \textbf{11.1} & \textbf{17.7} & \textbf{16.8} & \textbf{15.2} & \textbf{16.5} & \textbf{16.8} & \textbf{15.0} & \textbf{11.0} & \textbf{14.3} \\
    \bottomrule
  \end{tabular}
  \label{tab:qar_results_split}
\end{table}

\textbf{Easy and hard F1-scores. }
To better understand the performance differences between \(\anycq\) and SQL, we conduct a split evaluation over two categories of queries: (i) those with answers present in the observable graph (easy instances) and (ii) those requiring inference beyond the observable graph (hard instances). Results are reported in \Cref{tab:qar_results_split}.

On easy queries with a single free variable, SQL outperforms \(\anycq\), though the margin is modest (\(\leq8\%\) relative).
As the number of free variables increases, \(\anycq\) consistently surpasses SQL.
While a classical solver could in principle answer these queries given unlimited time, our evaluation threshold of 60 seconds leads SQL to time out on a non-trivial subset.
This indicates that \(\anycq\) can efficiently handle queries that impose substantial computational costs on symbolic engines.

On hard queries with no observable answers, SQL fails entirely, yielding \(0\%\) F1 across all splits.
In contrast, \(\anycq\) uses its integrated link predictor to infer missing edges while reasoning over complex query structures, successfully recovering answers across all settings. These findings validate our claims about the efficiency and effectiveness of \(\anycq\): it not only enables prediction in incomplete knowledge graphs but also serves as a practical query answering engine in scenarios where classical solvers struggle.

\begin{figure}[t]
    \centering
    \includegraphics[scale=0.30]{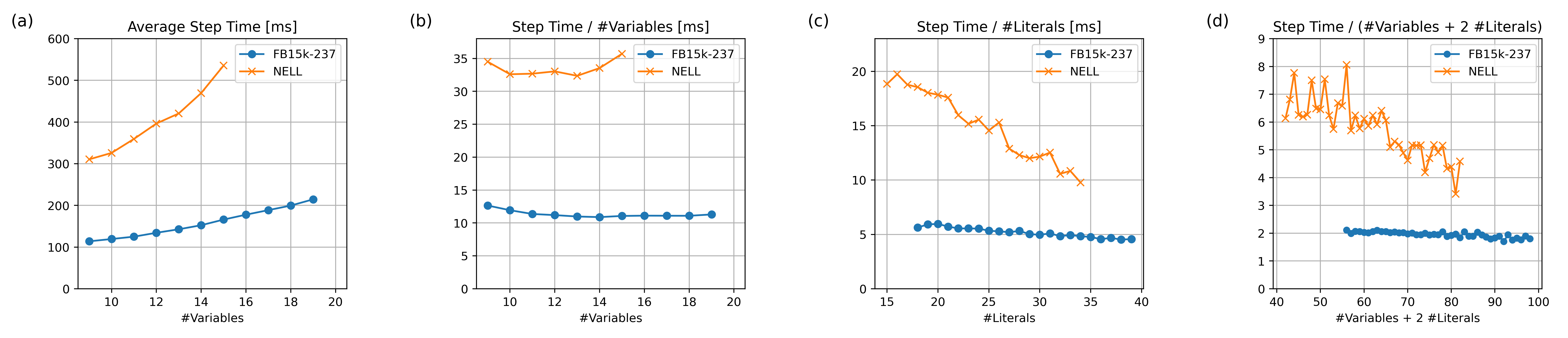}
    \caption{\(\anycq\) search step time analysis for queries of different complexities: a) average step time (AST) per the number of variables \(|\vec{y}|\), b) AST divided by the number of variables \(|\vec{y}|\), c) AST divided by the number of literals \(|Q|\), d) AST divided by \(|\vec{y}| + 2|Q|\), the complexity factor indicated by the theoretical analysis.}
    \label{fig:processing_times}
\end{figure}

\textbf{Computational complexity of \(\anycq\). }
Let \(Q = \exists \vec{y} . \Phi(\vec{y})\) be a conjunctive Boolean query over a knowledge graph \(G\).
Denote by \(|Q|\) the number of literals in \(Q\), let \(h\) be the maximum arity of~a~literal in \(Q\), and let \mbox{\(\vec{c} = (c_0, ..., c_s)\)} be the mentioned constants.
Then, the corresponding computational graph contains
\mbox{\((|V(G)| + 1) \cdot |\vec{y}| + 2 \cdot|\vec{c}| + |Q|\)}
vertices and at most \(|V(G)| \cdot (|\vec{y}| + h \cdot |Q|) + h\cdot |\vec{c}| \cdot |Q| \) edges.
Since \(\anycq\) processes this graph in linear time, the complexity of a single search step is:
\[
O \left(  |V(G)| \cdot (|\vec{y}| + h \cdot |Q|) + h\cdot |\vec{c}|\cdot |Q|  \right)
\]
Importantly, this complexity does not depend on the structure of the query graph of \(Q\) and scales only linearly with the sizes of the input formula and the KG \(G\).

We validate this linearity, evaluating the average \(\anycq\) search step time for queries of different sizes from the `3-hub' splits.
The results, presented in \Cref{fig:processing_times}, indicate that the empirical performance matches the theoretical analysis.
In particular, \Cref{fig:processing_times} b) and c) show that the processing time, divided by the number of variables \(|\vec{y}|\) or the number of literals \(|Q|\) in the input query, respectively, does not grow as the size of the query increases.
We even notice a slight decreasing trend, which we attribute to efficient GPU accelerations.
The difference between step times on FB15k-237-QAR and NELL-QAR remains consistent with the relative sizes of the underlying knowledge graphs.

\clearpage

\begin{table}[t]
  \centering
  % \small
  \scriptsize
  \caption{Recall on the easy samples from the QAR datasets with different \textsc{SQL} timeouts.}
  \setlength{\tabcolsep}{4pt}
  \begin{tabular}{cccccccccccccc}
    \toprule   
    \multirow{2}{*}{\textbf{Model}} & \multirow{2}{*}{\textbf{Timeout [s]}} & \multicolumn{4}{c}{\textbf{3-hub}} & \multicolumn{4}{c}{\textbf{4-hub}} & \multicolumn{4}{c}{\textbf{5-hub}} \\
    \cmidrule(lr){3-6}
    \cmidrule(lr){7-10}
    \cmidrule(lr){11-14}
     & & \(k\!=\!1\) & \(k\!=\!2\) & \(k\!=\!3\) & \textbf{total} & \(k\!=\!1\) & \(k\!=\!2\) & \(k\!=\!3\) & \textbf{total} & \(k\!=\!1\) & \(k\!=\!2\) & \(k\!=\!3\) & \textbf{total} \\
     \midrule
     & & \multicolumn{12}{c}{\textbf{FB15k-237-QAR}} \\
     \midrule
        \(\anycq\) & 60   & 83.3 & \textbf{84.2} & \textbf{71.6} & \textbf{80.5} & 79.1 & 71.1 & \textbf{72.1} & 74.3 & 74.0 & \textbf{67.0} & 52.3 & 65.0 \\
    \midrule
        \multirow{3}{*}{\textsc{SQL}}
            & 30 & 82.8 & 51.5 & 20.3 & 55.6 & 71.9 & 49.7 & 34.4 & 53.6 & 74.0 & 51.0 & 43.8 & 56.8 \\
         & 60 & 88.7 & 61.2 & 26.4 & 62.8 & 87.2 & 71.7 & 52.6 & 71.9 & 85.3 & 63.6 & 61.4 & 70.5 \\
         & 120 & \textbf{93.2} & 69.4 & 35.8 & 69.9 & \textbf{90.8} & \textbf{76.5} & 55.2 & \textbf{75.8} & \textbf{88.7} & \textbf{67.0} & \textbf{67.6} & \textbf{74.7} \\
     \midrule
     & & \multicolumn{12}{c}{\textbf{NELL-QAR}} \\
    \midrule
        \(\anycq\) & 60  & 93.0 & \textbf{89.4} & \textbf{86.5} & \textbf{90.7} & 89.6 & \textbf{88.8} & \textbf{79.6} & \textbf{87.1} & 89.4 & \textbf{84.5} & \textbf{81.3} & \textbf{86.1} \\
    \midrule
        \multirow{3}{*}{\textsc{SQL}} & 30 & 88.5 & 57.5 & 33.8 & 69.0 & 82.4 & 57.8 & 55.9 & 69.2 & 80.1 & 63.1 & 46.7 & 67.5 \\
         & 60 & 94.5 & 69.0 & 55.4 & 79.6 & 90.1 & 69.0 & 63.4 & 77.9 & 88.8 & 78.6 & 64.0 & 80.2 \\
         & 120 & \textbf{95.5} & 73.5 & 58.1 & 81.9 & \textbf{93.8} & 71.5 & 68.8 & 81.6 & \textbf{91.3} & 82.5 & 65.4 & 82.9 \\
    \bottomrule
  \end{tabular}
  \label{tab:sql_time_ablation}
\end{table}

\paragraph{Experiments over impact of the timeout threshold}
\label{app:timeout_threshold}
To further understand the impact of the timeout threshold on SQL performance, we evaluate it on all QAR splits using three different time limits: 30, 60, and 120 seconds.
Results are reported in \Cref{tab:sql_time_ablation}.
We observe that the SQL engine’s recall on \emph{easy} queries improves marginally as the timeout increases, particularly for queries with multiple free variables.
However, its performance consistently deteriorates with increasing arity, even under extended time limits. In contrast, \(\anycq\) maintains strong performance across all splits and consistently outperforms SQL on retrieving easy answers to high-arity queries.

\end{document}

%% file: math_commands.tex
\usepackage{amsmath,amsfonts,bm}

\def\eqref#1{equation~\ref{#1}}
\def\1{\bm{1}}

\def\rvh{{\mathbf{h}}}

\def\rvm{{\mathbf{m}}}

\def\rvo{{\mathbf{o}}}

\def\rvx{{\mathbf{x}}}
\def\rvy{{\mathbf{y}}}
\def\rvz{{\mathbf{z}}}

\def\rmE{{\mathbf{E}}}

\def\rmG{{\mathbf{G}}}

\def\rmM{{\mathbf{M}}}

\def\rmO{{\mathbf{O}}}

\def\rmU{{\mathbf{U}}}

\DeclareMathAlphabet{\mathsfit}{\encodingdefault}{\sfdefault}{m}{sl}
\SetMathAlphabet{\mathsfit}{bold}{\encodingdefault}{\sfdefault}{bx}{n}
\newcommand{\tens}[1]{\bm{\mathsfit{#1}}}

\def\gD{{\mathcal{D}}}

\def\gN{{\mathcal{N}}}

\def\sP{{\mathbb{P}}}

\newcommand{\etens}[1]{\mathsfit{#1}}

\DeclareMathOperator*{\argmax}{arg\,max}